\documentclass{article}

\usepackage[main, final]{neurips_2026}

\usepackage[utf8]{inputenc} % allow utf-8 input
\usepackage[T1]{fontenc}    % use 8-bit T1 fonts
\usepackage{url}            % simple URL typesetting
\usepackage{booktabs}       % professional-quality tables
\usepackage{amsfonts}       % blackboard math symbols
\usepackage{nicefrac}       % compact symbols for 1/2, etc.
\usepackage{microtype}      % microtypography
\usepackage{xcolor}         % colors

\usepackage{url}
\usepackage{graphicx}
\usepackage{subcaption}
\usepackage{booktabs}
\usepackage[table]{xcolor}
\usepackage{xspace}
\usepackage{algorithm}
\usepackage{algpseudocode}

\usepackage{booktabs}

\usepackage{wrapfig}

\usepackage{listings}
\usepackage{tcolorbox}
\tcbuselibrary{listings,skins,breakable}

\usepackage{microtype}
\usepackage{graphicx}
\usepackage{subcaption}
\usepackage{booktabs} % for professional tables

\usepackage[pdfencoding=auto, psdextra, pagebackref=true]{hyperref}
\hypersetup{
    colorlinks,
    linkcolor={blue!50!black},
    citecolor={blue!50!black},
    urlcolor={blue!80!black}
}

\renewcommand*{\backref}[1]{}

\renewcommand*{\backrefalt}[4]{%
  \ifcase #1%
  \or
    \space(Cited on p.~#2)%
  \else
    \space(Cited on pp.~#2)%
  \fi
}

\usepackage{amsmath}
\usepackage{amssymb}
\usepackage{mathtools}
\usepackage{amsthm}

\usepackage[capitalize,noabbrev]{cleveref}

\theoremstyle{plain}
\newtheorem{theorem}{Theorem}[section]
\newtheorem{proposition}[theorem]{Proposition}
\newtheorem{lemma}[theorem]{Lemma}
\newtheorem{corollary}[theorem]{Corollary}
\theoremstyle{definition}
\newtheorem{definition}[theorem]{Definition}
\newtheorem{assumption}[theorem]{Assumption}
\theoremstyle{remark}
\newtheorem{remark}[theorem]{Remark}

\usepackage[textsize=tiny]{todonotes}

\usepackage{enumitem}

\usepackage[toc,page]{appendix}

\usepackage{mathtools}

\usepackage{dsfont}

\usepackage{tcolorbox}

\usepackage{bbm}

\usepackage{footnote}

\usepackage{enumitem}

\usepackage{makecell}

\usepackage{xcolor}

\newcommand{\cov}{\operatorname{cov}}

\title{Prompt Interference in LLM Post-Training: \\Why Pass@k Optimization Can Degrade Pass@1}

\author{%
  Anas Barakat\\
  Singapore University of Technology and Design\\
  \texttt{anas\_barakat@sutd.edu.sg} \\
  \And
  Souradip Chakraborty \\
  Massachusetts Institute of Technology\\
  \texttt{soura24@mit.edu} \\
  \AND
  Khushbu Pahwa\thanks{Work is not related to position at Amazon.} \\
   Amazon \\
   \texttt{pahwa16khu@gmail.com} \\
  \And
  Amrit Singh Bedi \\
  University of Central Florida \\
  \texttt{amritbedi@ucf.edu} 
}

\begin{document}

\maketitle

\begin{abstract}
Pass@$k$ is a widely used performance metric for verifiable large language model tasks, including mathematical reasoning, code generation, and short-answer reasoning. It defines success if any of $k$ independently sampled solutions passes a verifier. This multi-sample inference metric has motivated inference-aware fine-tuning methods that directly optimize pass@$k$. However, prior work reports a trade-off: pass@$k$ improves while pass@1 degrades under such methods. This trade-off is practically important because pass@1 often remains a hard operational constraint due to latency and cost budgets, imperfect verifier coverage, and the need for a reliable single-shot fallback. 
We study the origin of this trade-off and provide a theoretical characterization of when pass@$k$ policy optimization can reduce pass@1 through gradient conflict induced by \emph{prompt interference}.
This interference arises because shared model parameters couple prompts, so improving performance on some prompts can adversely affect others. 
We show that pass@$k$ policy gradients can conflict with pass@1 gradients because pass@$k$ optimization implicitly reweights prompts toward low-success prompts; when these prompts are \emph{negatively interfering}, their upweighting can rotate the pass@$k$ update direction away from the pass@1 direction. 
We illustrate our findings with large language model experiments on verifiable mathematical reasoning and coding tasks. 
\end{abstract}

\section{Introduction}

\noindent \textbf{Pass@$k$ optimization.} In many verifiable tasks, such as code generation and short-answer math, a system can afford multiple response attempts for the same prompt and check the correctness of each response attempt with an automatic verifier (e.g., unit tests, exact-match) that can check the correctness of generated answers. This naturally leads to \emph{pass@$k$ inference}: sample $k$ independent candidate solutions and declare success if \emph{any} candidate is correct. The corresponding metric, pass@$k$, measures the probability that at least one of $k$ i.i.d.\ samples solves the prompt, and has become a standard performance measure when multiple attempts are available \citep{chen-et-al21eval-llms-code,li-et-al22alphacode,gehring-et-al25rlef}.  This “$k$ attempts with verification” setting suggests that the training objective should align with deployment: if inference uses $k$ attempts, training should directly optimize the same pass@$k$ objective rather than optimizing only the single-sample expected reward. Recent work develops policy-gradient estimators and practical training procedures that optimize pass@$k$, and demonstrates large gains in pass@$k$ performance on reasoning and code benchmarks \citep{tang-et-al25icml,walder-karkhanis25neurips}.

\vspace{-1em}
\noindent \textbf{Trade-off: optimizing pass@$k$ can reduce pass@1.}
A key complication is that pass@$k$ optimization does not necessarily improve single-shot pass@1 performance. Empirically, some recent works report objective-dependent trade-offs where pass@$k$ improves while pass@1 drops \citep{tang-et-al25icml,walder-karkhanis25neurips}. For example, Walder and Karkhanis \cite{walder-karkhanis25neurips} explicitly highlight a trade-off between pass@1 and pass@$k$ during code-generation fine-tuning (e.g. Tables 1 to 4 therein), and their results show that optimizing a pass@$k$ objective can lower pass@1 even as higher-$k$ metrics increase substantially (Table~1). This makes pass@$k$ optimization risky in settings where single-shot accuracy is a hard requirement.

\noindent \textbf{Why this degradation matters.}
Pass@1 remains operationally important: many deployments have limited latency or cost budgets, may not have reliable verifiers for every prompt, and often need a strong single-shot fallback even when retries are possible. In these regimes, a training procedure that improves pass@$k$ at the cost of pass@1 can reduce overall system reliability and complicate model selection and deployment.

\begin{figure}[t]
  \centering
  \includegraphics[width=\columnwidth]{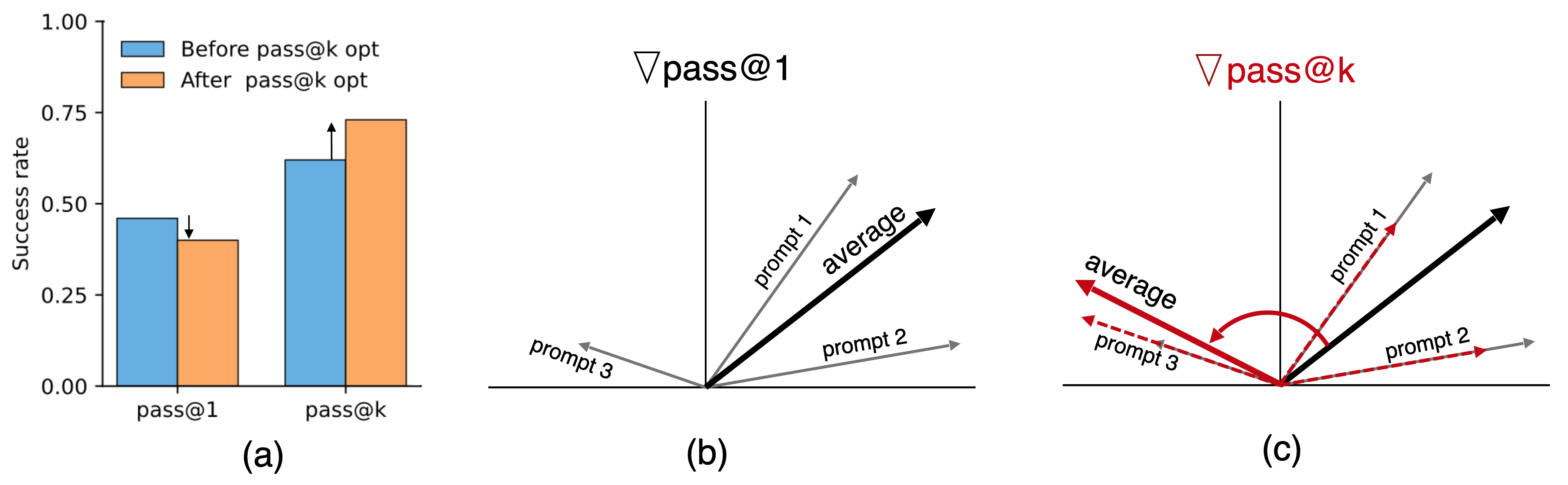}
  \caption{\textbf{(a) Empirical trade-off.} Under pass@$k$ policy optimization, pass@$k$ increases while pass@$1$ may decrease. We explain this empirically observed trade-off in \textbf{(b)} and \textbf{(c)}, which schematically illustrate the pass@$1$ and pass@$k$ ($k>1$) gradients for three prompts and their expectations in policy-parameter space. \textbf{(b) Pass@1 gradients with negatively interfering prompt.} This panel shows a setting in which prompt $3$ is \emph{negatively interfering} with prompts $1$ and $2$, i.e., the per-prompt pass@1 gradient for prompt $3$ has negative inner product with the per-prompt pass@$1$ gradients for prompts $1$ and $2$ (see Section~\ref{sec:prompt-interference}). Here, $\nabla$pass@$1$ denotes the population pass@$1$ gradient, given by the average (expectation over prompts; here under a uniform distribution) of the per-prompt pass@$1$ gradients. \textbf{(c) Pass@$k$ vs pass@1 gradient conflict.} Per-prompt pass@k gradients are scaled versions of the corresponding per-prompt pass@1 gradients (Eq.~\eqref{eq:pass@k-grad}). This reweighting amplifies the magnitude of the pass@$k$ gradient for the negatively interfering prompt $3$, causing the resulting population pass@$k$ gradient to conflict with the population pass@$1$ gradient (their inner product becomes negative, corresponding to an obtuse angle, shown here as approximately $120^\circ$). Hence, a policy update in the pass@$k$ gradient direction can increase pass@$k$ while decreasing pass@1.}
  \label{fig:main}
\end{figure}

\noindent \textbf{A per-prompt collinearity paradox.}
At a high level, one might expect pass@$k$ and pass@1 to trade off simply because they are different objectives. However, this intuition is incomplete in the pass@$k$ setting. For every fixed prompt, the pass@$k$ gradient is exactly a nonnegative scalar multiple of the pass@1 gradient. Thus, there is no per-prompt directional conflict between the two objectives. The empirical degradation of pass@1 must therefore arise only after aggregating across prompts.

\noindent\textbf{Open question.} Despite a growing adoption of pass@$k$ objectives, it is still not well understood \emph{why} pass@$k$ optimization can hurt pass@1, and \emph{when} we should expect this trade-off to appear. Without a principled explanation, it is difficult to design reliable inference-aware fine-tuning methods that deliver multi-attempt gains while preserving single-shot performance. This leads to our question:
\begin{quote}
\centering
\emph{If pass@$k$ and pass@1 are aligned on every individual prompt, when and why can pass@$k$ policy optimization degrade population pass@1 performance?}
\end{quote}

\noindent  \textbf{Our key insight.} 
We show that the conflict is instead a population-level effect. 
Compared to pass@1, optimizing the pass@$k$ objective induces an implicit \emph{prompt reweighting} toward prompts with lower success probability (i.e., prompts the current policy rarely solves). When these prompts contribute gradients that conflict with the population pass@1 gradient, upweighting them increases their influence on the pass@$k$ policy gradient update. Consequently, the pass@$k$ gradient can conflict with the pass@1 gradient direction. In this regime, pass@$k$ improves under pass@$k$ updates while pass@1 decreases (see Figure~\ref{fig:main} for a schematic explanation). We identify the prompts that can lead to gradient conflict as negatively interfering using the concept of \textit{prompt interference}. 

\noindent  Our main contributions are as follows:

\noindent  $\bullet$ \noindent\textbf{Prompt interference.} 
We formalize prompt interference through a gradient-similarity kernel over prompts: two prompts interfere negatively when an update improving one tends to reduce the other's success probability. 
The induced distribution over negatively interfering prompts is crucial to understand the pass@$k$ and pass@1 trade-off.

\noindent  $\bullet$ \textbf{Pass@$k$-specific gradient conflict characterization.} 
Although per-prompt pass@$k$ and pass@1 gradients are always positively collinear, we show that their population gradients can be conflicting in the sense that they can form an obtuse angle.
This implies that a policy update following pass@$k$’s policy gradient tends to increase pass@$k$ while decreasing pass@1. We provide a characterization of this gradient conflict by establishing an interpretable expression for the inner product between pass@$k$ and pass@1.  This characterization crucially uses the fact that per-prompt pass@$k$ gradients are \textit{weighted} pass@1 per-prompt policy gradients. The gradient conflict expression shows that gradient conflict occurs when negatively interfering prompt regions are upweighted by the implicit pass@$k$ reweighting and hence dominate the contribution of positively interfering prompts. 

\noindent  $\bullet$ \noindent\textbf{Sufficient conditions and influence of $k$.} Using our gradient conflict characterization, we provide sufficient conditions under which gradient conflict occurs. We further study the influence of the parameter $k$ and show that increasing $k$ encourages gradient conflict under some conditions in the relative probability of success in negatively versus positively interfering prompts. 

\noindent  $\bullet$ \noindent\textbf{Pass@1 degradation under pass@$k$ updates.} We prove that pass@1 decreases while pass@$k$ increases (simultaneously) under one-step pass@$k$ policy updates satisfying an explicit stepsize condition. To show this result, we establish and use smoothness of the pass@$k$ objective under a standard policy parameter regularity assumption.

\noindent  $\bullet$ \noindent\textbf{Empirical validation.} We illustrate our main findings on large language models for solving math reasoning questions and code generation tasks.

\noindent\textbf{Relation to multi-task gradient conflict.}
Our notion of prompt interference is related to ray interference and negative transfer in multi-task learning \cite{schaul-et-al19ray-interference}, as well as gradient-conflict methods such as gradient surgery \cite{yu-et-al20gradient-surgery-mtl}. However, the phenomenon studied here is not a direct instance of generic task-gradient conflict. In standard multi-task optimization, one typically studies arbitrary task objectives whose gradients may already disagree. In contrast, pass@1 and pass@$k$ are two structured inference-time objectives induced by the same prompt-wise success probabilities, and for every fixed prompt their gradients are positively collinear. Thus the central question is different: how can two per-prompt aligned objectives become conflicting after averaging over prompts? Standard per-task gradient conflict alone cannot explain pass@1 degradation. We show that the answer is the pass@$k$-specific weight $w_{k,\theta}(x)=k(1-p_\theta(x))^{k-1}$, which shifts the effective prompt distribution toward low-success prompts. 
Gradient conflict arises only when this structured reweighting amplifies prompts that negatively interfere with population pass@1 improvement. Our contribution is therefore to characterize when this pass@$k$-induced distribution shift turns prompt interference into pass@1 degradation. See Appendix~\ref{sec:related-work} for an extended related work discussion.

\section{Preliminaries}

In this section, we formally define the pass@$k$ objective and recall its policy gradients as used in pass@$k$ policy optimization. 
We consider a setting where a language model is represented as a stochastic policy $\pi_{\theta}$ parameterized by weights $\theta \in \mathbb{R}^d$ over a set of responses~$\mathcal{Y}$ given a set of prompts denoted by~$\mathcal{X}$ (e.g., math questions). The conditional probability distribution of the model’s output responses when prompted with $x$ is $\pi_{\theta}(\cdot|x).$ We denote by $\mathcal{D}$ the distribution over prompts.
 
\noindent\textbf{Pass@$k$ objective.}  
Given a prompt $x$ sampled from a distribution~$\mathcal{D}$, a single `attempt' is a response~$y \sim \pi_{\theta}(\cdot|x).$ 
We suppose we have access to a binary reward verifier $r(x,y) \in \{0,1\}$ which encodes correctness of response~$y$ for prompt $x$. The per-prompt success probability for any $x \in \mathcal{X}$ and any policy parameter~$\theta \in \mathbb{R}^d$ will be denoted: 
\begin{equation}
\label{eq:ptheta-def}
p_{\theta}(x) := \mathbb{E}_{y \sim \pi_{\theta}(\cdot|x)}[r(x,y)] = \mathbb{P}_{y \sim \pi_\theta(\cdot \mid x)}(r(x,y)=1)\,,
\end{equation}
where the last identity follows from the fact that rewards are binary ($r(x,y) \in \{0,1\}$). 
The pass@$k$ objective is defined as the probability that at least one response is correct among~$k$ responses, i.e.,  
\begin{equation}
\label{eq:pass@k}
J_k(\theta) 
:= \mathbb{P}\left(\max_{1 \le i \le k} r(x,y_i) = 1\right)
= \mathbb{E}_{x \sim \mathcal{D}}[1- (1- p_{\theta}(x))^k]\,,
\tag{pass@k}
\end{equation}
where $x \sim \mathcal{D}, y_i \sim \pi_{\theta}(\cdot|x)$ for $i = 1, \cdots k$ are $k$ independent response samples.  
Hence pass@$k$ can be written using the nonlinear transformation $f_k(p) := 1- (1-p)^k$ as 
$J_k(\theta) = \mathbb{E}_{x \sim \mathcal{D}}[f_k(p_{\theta}(x))]\,.$ 
It immediately follows that pass@$k$ is an increasing function in $k$: 
$\forall m \geq k \geq 1, \forall \theta \in \mathbb{R}^d, J_m(\theta) \geq J_k(\theta)\,.$ 
Moreover, pass@$k$ can be bounded by pass@1 using Jensen's inequality and concavity of $f_k$ w.r.t. its argument: for all $k \geq 1, \theta \in \mathbb{R}^d$, 
$J_1(\theta) \leq J_k(\theta) \leq \min\left(1, f_k(J_1(\theta))\right)\,.$
Note that under our notation $J_1(\theta) = \mathbb{E}_{x \sim \mathcal{D}}[p_{\theta}(x)].$

\noindent\textbf{Pass@$k$ policy gradients.} Pass@$k$ optimization can be performed using pass@$k$ policy gradients. It follows from the chain rule that the gradient of pass@$k$ w.r.t. the policy parameters is given by: 
\begin{equation}
\label{eq:pass@k-grad}
\nabla J_k(\theta) = \mathbb{E}_{x \sim \mathcal{D}}[ w_{k}(p_{\theta}(x)) \nabla p_{\theta}(x)]\,, w_{k}(p) := k (1- p)^{k-1},
\end{equation}
as noted in a few recent works \citep{mahdavi-et-al25pg-passk,yu25}. Pass@$k$ gradients reweight prompts by
\begin{equation}
\label{eq:pass@k-weights-formula}
w_{k,\theta}(x) := w_k(p_{\theta}(x)) = k (1- p_{\theta}(x))^{k-1}\,.
\end{equation}
This weighting emphasizes low-probability of success prompts and downweights prompts with high probability of success.  
To compute pass@$k$ gradients in \eqref{eq:pass@k-grad}, note that the gradient~$\nabla p_{\theta}(x)$ of the prompt-wise probability of success  
for any prompt $x \in \mathcal{X}$ and any policy parameter $\theta \in \mathbb{R}^d$ immediately follows from \eqref{eq:ptheta-def} and is given by:
\begin{equation}
\label{eq:pass@1pg}
\nabla p_{\theta}(x) = \mathbb{E}_{y \sim \pi_{\theta}(\cdot|x)}[r(x,y) s_{\theta}(x,y)]\,,
\end{equation}
where $s_{\theta}(x,y) := \nabla \log \pi_{\theta}(y|x)$ denotes the score function. 

\noindent This gradient expression has been recently used in \cite{tang-et-al25icml,walder-karkhanis25neurips} for instance. 
Per-prompt pass@$k$ and pass@1 are positively collinear as $\nabla J_k(x;\theta) := w_{k,\theta}(x) \nabla p_{\theta}(x) = w_{k,\theta}(x) \nabla J_1(x;\theta)$ and $w_{k,\theta}(x) \geq 0$. However we will crucially see that pass@$k$ and pass@1 as \textit{expectation over prompts} are not necessarily positively collinear. This discrepancy, which was overlooked in the literature, highlights the importance of the expectation over prompts and the structure of per-prompt pass@1 gradients. This motivates our next main section in which we formalize the concept of \textit{prompt interference} capturing the potential conflict between the diversity of prompts in the underlying prompt distribution and the shared policy parameters across prompts.   

\section{Prompt Interference}
\label{sec:prompt-interference}

In this section, we formalize the phenomenon of \textit{prompt interference} which will play an important role in explaining later in the paper why pass@1 can degrade under pass@$k$ policy updates. This section is focused on understanding the structure of pass@1 policy gradients over the set of prompts. The role of the pass@$k$ objective and its influence on optimization will be discussed in section~\ref{sec:grad-conflict}. 

The pass@1 objective is defined as the expectation of the per-prompt probability of success with respect to the distribution of prompts (see \ref{eq:pass@k} with $k=1$). As a function of policy parameters, this pass@1 objective induces a shared-parameter coupling across prompts: the same policy parameters are shared across the entire distribution of prompts which might be diverse. 
As a consequence, a (shared) policy parameter update can improve the probability of success ($p_{\theta}(x)$) on some prompts while it can degrade it on other prompts. When this negative transfer occurs, we say that there is negative prompt interference. We provide a formal definition of this phenomenon in this section.  

\noindent To define prompt interference, we investigate the structure of the pass@1 policy gradient, looking at the \textit{per-prompt} pass@1 policy gradients and comparing them across prompts. We define the following natural prompt gradient similarity measure $\kappa_{\theta}: \mathcal{X} \times \mathcal{X} \to \mathbb{R}$ defined for any policy parameter~$\theta \in \mathbb{R}^d$ and any prompts $x, x' \in \mathcal{X}$ by: 
\begin{equation}
\label{eq:kernel-def-1}
\kappa_{\theta}(x,x') := \langle \nabla p_{\theta}(x),  \nabla p_{\theta}(x')\rangle\,.
\end{equation}

\noindent This kernel captures similarities between prompts in terms of their pass@1 gradients which encode their ability to improve the probability of success. Using the policy gradient expressions of the gradients of pass@1 in \eqref{eq:pass@1pg}, we obtain: 
\begin{equation}
\label{eq:kernel-def-2}
\kappa_{\theta}(x,x') =
\mathbb{E}_{\makecell{\scriptstyle y \sim \pi_{\theta}(\cdot |x) 
\scriptstyle{y' \sim \pi_{\theta}(\cdot |x')}}}
\bigl[\delta(x,y,x',y') \langle s_{\theta}(x,y), s_{\theta}(x',y') \rangle\bigr],
\end{equation}
where $\delta(x,y,x',y') := r(x,y)\, r(x',y')$ which is equal to~$1$ if and only if both $y, y'$ are correct completions for $x, x'$ respectively (as rewards are binary). 
The kernel informs on whether improving pass@1 on one prompt tends to also improve pass@1 on another prompt under the shared policy parameters. To see this, consider a policy update $\theta' = \theta + \eta \nabla p_{\theta}(x)$ with a small stepsize~$\eta > 0$ in the direction $\nabla p_{\theta}(x)$ improving the probability of success on $x$. A first-order Taylor expansion of $\theta \mapsto p_{\theta}(x')$ gives $p_{\theta^+}(x') \simeq p_{\theta}(x') + \eta \langle \nabla p_{\theta}(x'),  \nabla p_{\theta}(x) \rangle
= p_{\theta}(x') + \eta \kappa_{\theta}(x', x)\,.$ 
This leads us to the following prompt interference definition. 
\begin{tcolorbox}[colframe=white!, top=2pt,left=2pt,right=2pt,bottom=2pt] 
\begin{definition}[Prompt interference]
\label{def:prompt-interference}
Two prompts $x$ and $x'$ \emph{interfere positively} if $\kappa_{\theta}(x',x)>0$: a small update that increases $p_\theta(x)$ also tends to increase $p_\theta(x')$. They \emph{interfere negatively} if $\kappa_{\theta}(x',x)<0$, in which case increasing $p_\theta(x)$ tends to decrease $p_\theta(x')$.
\end{definition}
\end{tcolorbox}
This notion parallels ray interference in multi-task learning~\citep[Definition~1]{schaul-et-al19ray-interference}, but here the gradients are prompt-wise pass@1 gradients in LLM post-training. See Appendix~\ref{sec:related-work} for further discussion on related work. We next illustrate prompt interference on a toy example.

\noindent\textbf{Prompt interference in a toy example.} In Appendix~\ref{sec:toy-example}, we consider a minimal contextual-bandit abstraction of verifier-based multi-sample LLM evaluation as a running example to illustrate (negative) prompt interference and the pass@$k$-pass@1 trade-off. The detailed description of this example and a visualization of the cosine kernel heatmap are given in Appendix~\ref{sec:toy-example} (see Figure~\ref{fig:cosine-kernel-heatmap}). In this example, we demonstrate that pass@1 may decrease while pass@$k$ increases under pass@$k$ optimization. We show that negative prompt interference can induce a conflict between pass@1 and pass@$k$ gradients when it is amplified by the induced pass@$k$ prompt reweighting. In the next section, we go beyond our toy example to characterize this conflict between pass@$k$ and pass@1 gradients in general.

\section{Pass@$k$ and Pass@1 Gradients Conflict}
\label{sec:grad-conflict}

In this section, we exploit the observation that pass@$k$ policy optimization induces an implicit prompt reweighting, biased toward upweighting low-success prompts. To connect with the previous section, if this reweighting coincides with upweighting negatively interfering prompts, then pass@$k$ policy updates may result in conflicting pass@$k$ and pass@1 gradients, leading to pass@1 degradation under pass@$k$ policy optimization updates. We make this chain of arguments precise in this section. 

\subsection{Gradient conflict: Reweighting meets interference}

To explain why optimizing pass@$k$ can hurt pass@1, we study a key quantity: the angle between the policy gradients of pass@$k$ and pass@1 respectively. We show that this angle can be obtuse and we refer to this setting as a case of \textit{conflicting gradients}.  

\noindent This observation might seem a little surprising from the theory viewpoint at first glance. Indeed, 
the per-prompt gradient $\nabla J_k(x;\theta)$ of pass@$k$ is always a positive multiple of the per-prompt policy gradient $\nabla J_1(x;\theta)$ of pass@1 since $\nabla J_k(x;\theta) = w_{k,\theta}(x) \nabla J_1(x;\theta)$ and $w_{k,\theta}(x) \geq 0$ for any prompt $x \in \mathcal{X}\,.$  Nevertheless, the population gradients $\nabla J_k(\theta) = \mathbb{E}[\nabla J_k(x;\theta)]$ and $\nabla J_1(\theta) = \mathbb{E}[\nabla J_1(x;\theta)]$ need not form an acute angle as their inner product can be negative as we will see.

\noindent The following proposition provides an explicit expression of this inner product to characterize the gradients conflict. To state this result, we introduce the prompt-wise gradient agreement score for pass@1 defined for any policy parameter~$\theta \in \mathbb{R}^d$ and any prompt~$x \in \mathcal{X}$ as follows: 
\begin{equation}
\label{eq:alignment-score}
a_{\theta}(x) := \langle \nabla J_1(x;\theta), \nabla J_1(\theta) \rangle = \mathbb{E}_{x' \sim \mathcal{D}}[\kappa_{\theta}(x,x')]\,.
\end{equation}
This agreement score quantifies the prompt-level gradient interference with the average pass@1 objective. If $a_{\theta}(x) > 0$, that prompt's pass@1 gradient points in a direction that also improves average pass@1. If $a_{\theta}(x) < 0$, it is anti-aligned and improving that prompt's pass@1 tends to decrease population pass@1 under the shared policy parameters~$\theta\,.$
Using this agreement score, our characterization is as follows. 

\begin{tcolorbox}[colframe=white!, top=2pt,left=2pt,right=2pt,bottom=2pt] 
\begin{proposition}[Gradients conflict characterization]
\label{prop:grad-conflict}
For any $k \geq 1$ and any $\theta \in \mathbb{R}^d$, 
\begin{align}
\langle \nabla J_k(\theta), \nabla J_1(\theta) \rangle 
= \mathbb{E}_{x \sim \mathcal{D}}[w_{k,\theta}(x) a_{\theta}(x)]
= \mathbb{E}[w_{k,\theta}(x)] \cdot \|\nabla J_1(\theta)\|^2 + \cov(w_{k,\theta}(x), a_{\theta}(x))  \label{eq:conflict-1}\,.
\end{align}
In particular, $\langle \nabla J_k(\theta), \nabla J_1(\theta) \rangle < 0$ is equivalent to each one of the following conditions:\\ 
\textit{(i)} $\mathbb{E}_{x \sim \mathcal{D}}[w_{k,\theta}(x) a_{\theta}(x)] < 0\,,$ and \textit{(ii)} $\cov(w_{k,\theta}(x), a_{\theta}(x)) < - \mathbb{E}[w_{k,\theta}(x)] \cdot \|\nabla J_1(\theta)\|^2\,.$ 
\end{proposition}
\end{tcolorbox}

\noindent\textbf{Discussion of conditions.} We provide two interpretations explaining gradients conflict and we refer the reader to Figure~\ref{fig:agreement-score} below for illustrations on the toy example of section~\ref{sec:toy-example}, see also Figure~\ref{fig:front-page} in Appendix~\ref{sec:toy-example} for a contour plot of pass@1 and pass@$k$ objectives in the policy parameter space showing the gradient conflict region in 2D.  

\noindent \textit{(i) Prompt distribution shift: pass@$k$ reweights prompts.}  
Recalling that the weights $w_{k,\theta}(x)$ are nonnegative and assuming $\mathbb{E}_{\mathcal D}[w_{k,\theta}(x)]>0$, define the reweighted prompt distribution $\tilde{\mathcal{D}}_{k,\theta}(dx) \, \propto w_{k,\theta}(x) \mathcal{D}(dx)$ 
inducing the expectation $\mathbb{E}_{\tilde{\mathcal{D}}_{k,\theta}}[f(x)] = \frac{\mathbb{E}_{\mathcal{D}}[w_{k,\theta}(x)f(x)]}{\mathbb{E}_{\mathcal{D}}[w_{k,\theta}(x)]}$ for any measurable function~$f$ of the prompt random variable~$x$.
Note that the denominator is nonzero as the weights are not equal to zero a.e.
The distribution $\tilde{\mathcal{D}}_{k,\theta}$ places higher mass on prompts with smaller success probability $p_{\theta}(x)$ (hard prompts) since $w_{k,\theta}(x)$ is decreasing in $p_{\theta}(x).$ Therefore, we can rewrite \eqref{eq:conflict-1} as follows: 
\begin{equation}
\langle \nabla J_k(\theta), \nabla J_1(\theta) \rangle 
= \mathbb{E}_{\mathcal{D}}[w_{k,\theta}(x)]\cdot\mathbb{E}_{\tilde{\mathcal{D}}_{k, \theta}}[a_{\theta}(x)]\,,
\end{equation}
and since $\mathbb{E}_{\mathcal{D}}[w_{k,\theta}(x)] > 0$ ( $w_{k,\theta}(x) \geq 0$ and not equal to zero a.e.\footnotemark[1])\footnotetext[1]{\textit{Degenerate weights:}  For $k \geq 2, \mathbb{E}[w_{k,\theta}(x)] = 0$ occurs exactly when $p_{\theta}(x) = 1$ for almost every prompt $x$. This means that the model is perfect for almost every prompt and there is no need to perform policy optimization. In this case, $\nabla J_k(\theta) = 0$ and the normalized reweighted distribution is undefined. The unnormalized identity in Proposition~\ref{prop:grad-conflict} remains valid without this condition.}, we obtain the following equivalence between $\langle \nabla J_k(\theta), \nabla J_1(\theta) \rangle < 0$ and $\mathbb{E}_{\tilde{\mathcal{D}}_{k, \theta}}[a_{\theta}(x)] < 0\,.$
\begin{figure}[b!]
  \centering
  \includegraphics[width=0.8\columnwidth]{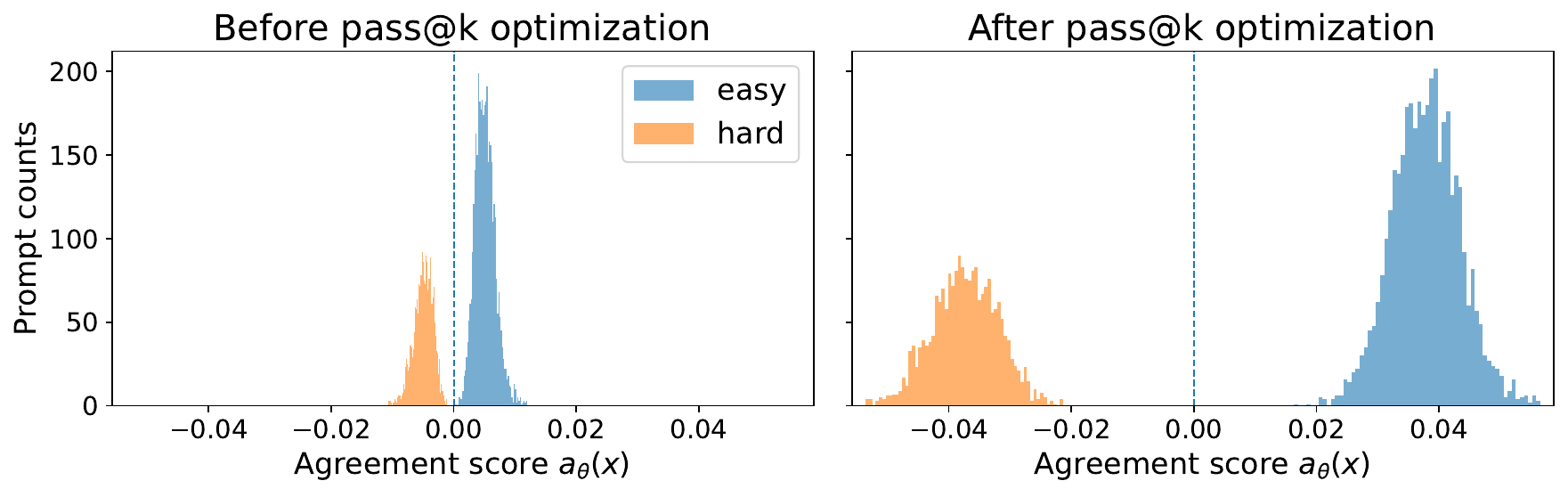}
  \caption{The prompt distribution has prompts with negative and positive agreement scores~$a_{\theta}(x)$ \eqref{eq:alignment-score} on the toy example (Appendix~\ref{sec:toy-example}). Prompts with negative scores are responsible for gradient conflict and hence pass@1 decrease. After pass@$k$ optimization, negative scores get even more negative.}
  \label{fig:agreement-score}
\end{figure}
This characterization shows that conflict of gradients of pass@1 and pass@$k$ occurs if and only if the average of the agreement score $a_{\theta}(x)$ under the pass@$k$ induced prompt distribution $\tilde{\mathcal{D}}_{k, \theta}$ is negative. In contrast note that 
$\mathbb{E}_{\mathcal{D}}[a_{\theta}(x)] = \|\nabla J_1(\theta)\|^2 \geq 0\,.$ Pass@$k$ changes the prompt weighting from $\mathcal{D}$ to $\tilde{\mathcal{D}}_{k, \theta}$ and the average agreement under this shifted distribution can become negative. 

\noindent In summary, we obtain the following causal chain: pass@$k$ induces a prompt reweighting toward harder prompts (i.e. lower $p_{\theta}(x)$ since $w_{k,\theta}(x)$ decreases in $p_{\theta}(x)$). If hard prompts have negative agreement score $a_{\theta}(x)$ (i.e. per-prompt pass@1 gradients are conflicting with the pass@1 gradient) then the reweighted average agreement becomes negative and the gradient of pass@$k$ and pass@1 form an obtuse angle as a consequence, i.e. they are conflicting.   
 
\noindent  \textit{(ii) Covariance between weights and agreement scores.} 
The covariance~$\cov(w_{k,\theta}(x), a_{\theta}(x))$ captures the coupling between `how much pass@$k$ emphasizes a prompt' (via weight $w_{k,\theta}(x)$) and `whether that prompt helps pass@1' (via alignment score~$a_{\theta}(x)$). The covariance term quantifies whether pass@$k$’s hardness reweighting concentrates gradient mass on prompts whose success-probability gradients are anti-aligned with the pass@1 gradient. Conflict happens when this anti-aligned hard-prompt mass overwhelms the positive term. See Appendix~\ref{appx:covar-condition} for more details.

\begin{remark}
For a more general discussion of the conflict between pass@$k$ and pass@$m$ gradients for any $k, m \geq 1$ using our kernel viewpoint, see Appendix~\ref{sec:kernel-grad-conflict-k-m}. 
\end{remark}

\subsection{Sufficient conditions and threshold in $k$}

We now provide sufficient conditions for pass@$k$ and pass@1 gradients conflict, as a corollary of Proposition~\ref{prop:grad-conflict}. 
First, we make a standard smoothness assumption on the policy parameterization. 
\begin{assumption}[Lipschitz and smooth policy]
\label{as:E-LS}
There exist $G, F > 0$ s.t. for all prompts $x \in \mathcal{X}$ and all $\theta\in\mathbb R^d$, we have the expected score bounds: $\mathbb{E}_{y \sim \pi_{\theta}(\cdot|x)}[\|\nabla \log \pi_{\theta}(y|x)\|^2] \leq G^2$ and $\mathbb{E}_{y \sim \pi_{\theta}(\cdot|x)}[\|\nabla^2 \log \pi_{\theta}(y|x)\|] \leq F\,.$ 
\end{assumption}

\noindent 
This is a standard assumption used in the convergence analysis of policy gradient methods, see, e.g., \cite{yuan-et-al22vanilla-pg,papini-et-al22smoothing}, and satisfied by Gaussian policies with bounded feature vectors and softmax tabular policies.

\noindent To quantify the magnitude of negative interference and its prevalence under the distribution~$\mathcal{D}$, we define the $m$-strongly negatively interfering prompt set with margin~$m>0$:    
\begin{equation}
\mathcal{X}_{-}(\theta,m) := \{x \in \mathcal{X}: a_{\theta}(x) \leq - m\}\,,
\end{equation}
and denote by $q_{\theta} := \mathcal{D}(\mathcal{X}_{-}(\theta,m)) \in [0,1]$ its probability mass under the distribution~$\mathcal{D}$ of prompts.   
Then, the expected pass@$k$ weight contributions inside and outside the $m$-strongly interfering set~$\mathcal{X}_{-}(\theta, m)$ are given by: 
\begin{equation}
\label{eq:aggregated-weights}
W_{-}(k,\theta) := \mathbb{E}[w_{k,\theta}(x) \mathds{1}_{\{x \in \mathcal{X}_{-}(\theta, m)\}}]\,,\quad 
W_{+}(k,\theta) := \mathbb{E}[w_{k,\theta}(x) \mathds{1}_{\{x \notin \mathcal{X}_{-}(\theta, m)\}}]\,. 
\end{equation}
The following corollary of Proposition~\ref{prop:grad-conflict} shows that gradients of pass@1 and pass@$k$ conflict when the contribution of pass@$k$ weights in negatively interfering prompt regions dominates its counterpart for positively interfering prompts.  
\begin{tcolorbox}[colframe=white!, top=2pt,left=2pt,right=2pt,bottom=2pt] 
\begin{corollary}[Dominating negatively interfering prompts]
\label{cor:dominating-neg-interfer-prompts}
Let Assumption~\ref{as:E-LS} hold. Suppose that $q_{\theta} > 0$ and let $k \geq 2.$ Then for any $\theta \in \mathbb{R}^d,$
$\langle \nabla J_k(\theta), \nabla J_1(\theta)\rangle \leq - \delta(\theta)$ where $\delta(\theta) := m W_{-}(k,\theta) - G^2 W_{+}(k,\theta)$. If in addition $\delta(\theta) > 0$, then $\langle \nabla J_k(\theta), \nabla J_1(\theta)\rangle \leq - \delta(\theta) < 0\,.$ 
\end{corollary}
\end{tcolorbox}

\noindent\textbf{Discussion of assumption.} 
Some negative agreement is necessary for gradient conflict. 
Indeed, if instead we have for instance for almost every $x \in \mathcal{X}, a_{\theta}(x) \geq 0$, then it follows from \eqref{eq:conflict-1} that $\langle \nabla J_k(\theta), \nabla J_1(\theta)\rangle \geq 0$ as the weights $w_{k,\theta}(x)$ are also always nonnegative. Negative agreement score is a necessary condition for gradient conflict.   
The gradient conflict condition of Corollary~\ref{cor:dominating-neg-interfer-prompts} depends on the relative weights ratio $W_{-}(k,\theta)/W_{+}(k,\theta)$ in negative vs positive prompt agreement score regions. This ratio is a function of both $k$ and the probability of success~$p_\theta(x)$~\eqref{eq:ptheta-def} under policy~$\pi_{\theta}\,.$ We now exploit the closed form of the pass@$k$ weights~$w_{k,\theta}(x) := k(1-p_{\theta}(x))^{k-1}$ to evaluate the ratio under assumptions on the prompt-wise probabilities of success. The next result shows that beyond a threshold $k^{\star}$, the  pass@$k$ and pass@1 gradients form an obtuse angle. We make a separation assumption on the probability of success over prompts: the probability of success for $m$-strongly interfering prompts is supposed to be uniformly smaller than that of the remaining prompts. Intuitively, this boils down to assume that negatively interfering prompts are relatively harder to solve than the rest of the prompts. This assumption may not always hold.

\begin{tcolorbox}[colframe=white!, top=2pt,left=2pt,right=2pt,bottom=2pt] 
\begin{proposition}[Influence of $k$]
\label{prop:phase-transition-in-k}
In the setting of corollary~\ref{cor:dominating-neg-interfer-prompts}, assume there exist $\varepsilon, \delta, q \in (0,1)$ s.t. 
$\sup_{x \in \mathcal{X}_{-}(m,\theta)} p_{\theta}(x) \leq \varepsilon < \delta \leq \inf_{x \notin \mathcal{X}_{-}(m,\theta)} p_{\theta}(x), $
and $q_{\theta} \geq q$ for any $\theta$, then 
$\langle \nabla J_k(\theta), \nabla J_1(\theta)\rangle \leq - \Delta$ where $\Delta = k [(1-\varepsilon)^{k-1} m q - (1-\delta)^{k-1} G^2 (1-q)].$
If in addition  
$k > k^{\star} := 1 + \frac{\log(\frac{(1-q) G^2}{q m})}{\log(\frac{1-\varepsilon}{1-\delta})}\,,$
then $\langle \nabla J_k(\theta), \nabla J_1(\theta) \rangle < 0\,.$
\end{proposition}
\end{tcolorbox}
If $k$ is large enough then gradients of pass@$k$ and pass@1 will be conflicting and pass@1 will degrade while pass@$k$ will improve under a small step size policy update following pass@$k$'s gradient. 
The threshold $k^{\star}$ becomes smaller when the probability mass of the negatively interfering region is larger. 
Hence, under larger negative prompt interference, the pass@$k$-pass@1 conflict becomes more severe. 

\subsection{Pass@1 degradation under pass@$k$ updates}

\begin{minipage}[t]{0.4\linewidth}
\vspace{0pt}
In this section, we show and quantify pass@1 decrease and pass@$k$ increase under pass@$k$ policy optimization, exploiting the conflict between pass@$k$ and pass@1 gradients shown in the previous sections. To quantify improvement and degradation under policy updates, we prove smoothness of the pass@$k$ objective for any $k \geq 1$ under Assumption~\ref{as:E-LS} (see Lemma~\ref{lem:smoothness-pass@k} in appendix~\ref{appx:lemma-smoothness}). Using this smoothness property, we show that gradient conflict results in provable pass@1 decrease under pass@$k$ optimization. 
\end{minipage}
\hfill
\begin{minipage}[t]{0.55\linewidth}
\vspace{0pt}
\centering
\includegraphics[width=0.9\linewidth]{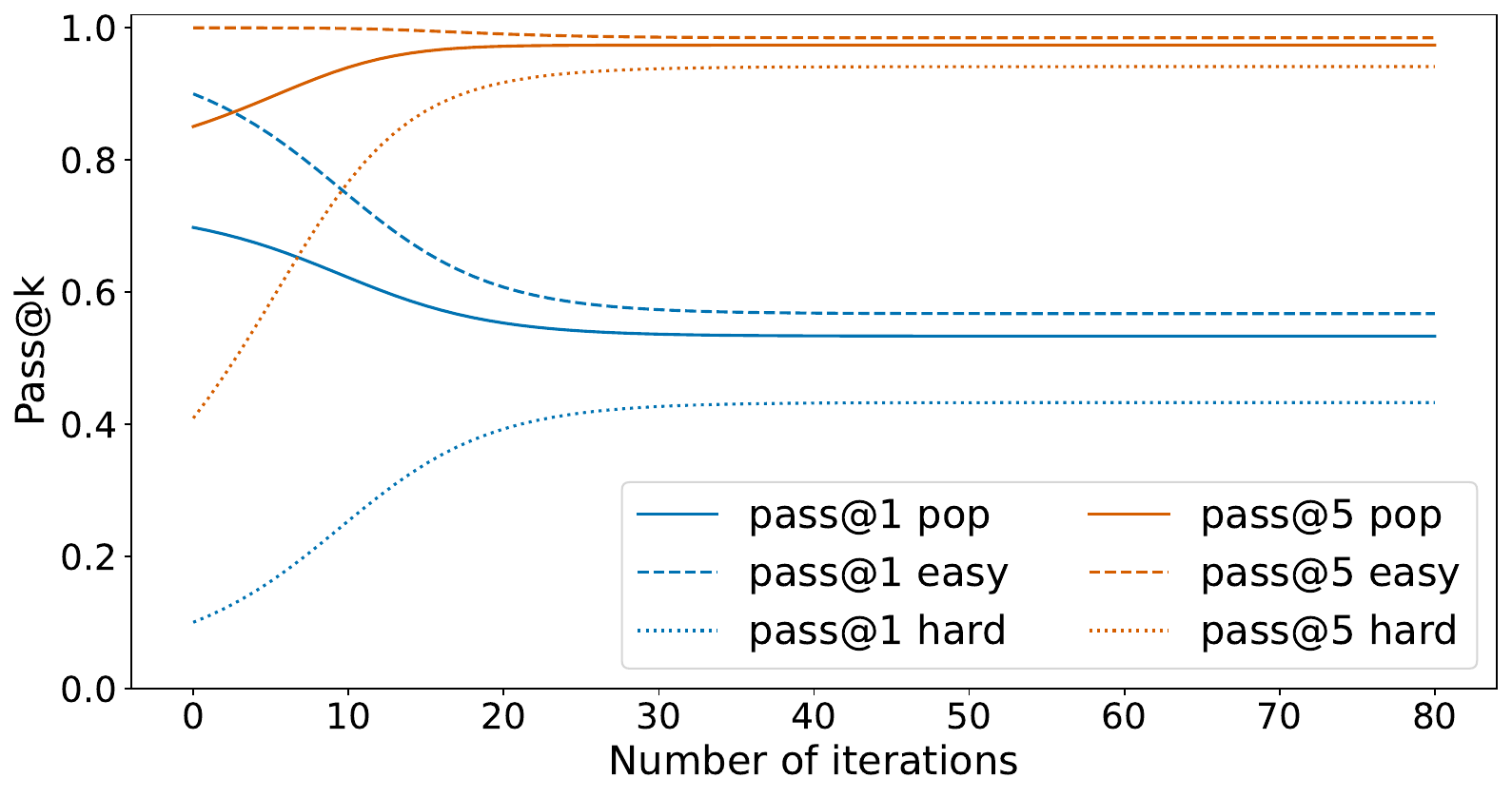}
\captionof{figure}{Pass@5 increases while pass@1 decreases under pass@5 optimization in example of section~\ref{sec:toy-example}, `pop' refers to `population' i.e. pass@$k$ as defined in \eqref{eq:pass@k}, `easy' (resp. `hard') means expectation is only taken over prompts labeled as easy (resp. hard).}
\label{fig:pass@k-pass@1}
\end{minipage}

\begin{tcolorbox}[colframe=white!, top=2pt,left=2pt,right=2pt,bottom=2pt] 
\begin{proposition}[Pass@$1$ decrease under pass@$k$ updates]
\label{prop:pass@1-degradation}
In the setting of corollary~\ref{cor:dominating-neg-interfer-prompts}, set $\theta^+ = \theta + \eta \nabla J_k(\theta)$ where $k \geq 2$ and $\eta > 0$. If $0 < \eta < \min(\frac{\delta(\theta)}{C_2}, \frac{1}{L_k})$ where $C_2 := k^2 G^2 L_1/2$, 
\begin{equation}
J_1(\theta^+) \leq J_1(\theta) - \eta \delta(\theta) + C_2 \eta^2 < J_1(\theta)\,,\quad
J_k(\theta^+) \geq  J_k(\theta) + \frac{\eta}{2} \|\nabla J_k(\theta)\|^2 \geq  J_k(\theta)\,.
\end{equation}
\end{proposition}
\end{tcolorbox}

We conclude this section with additional practical insights derived from our theoretical investigation. 

\begin{tcolorbox}[colframe=white!, top=2pt,left=2pt,right=2pt,bottom=2pt,after skip=0pt] 
\noindent\textbf{Practical implications.}
Our theory suggests three diagnostics:\\

(i) \emph{Check interference:} pass@$k$ optimization is risky when low-success prompts have negative agreement with the pass@1 gradient.\\

(ii) \emph{Tune $k$ cautiously:} larger $k$ strengthens the reweighting and can amplify conflict \textit{when} such prompts are negatively interfering.\\

(iii) \emph{Inspect prompt subsets:} two datasets with the same average pass@1 can behave differently; what matters is whether the low-success prompts upweighted by pass@$k$ are aligned or anti-aligned with population pass@1 improvement. This suggests clipped/interference-aware reweighting or multi-objective updates as mitigation.
\end{tcolorbox}

\section{Experiments}
\label{sec:experiments}

We empirically test whether the pass@$k$ objective can induce {gradient conflict} with pass@1 on math reasoning, as predicted by Proposition~\ref{prop:grad-conflict}. We provide further evidence on code generation (HumanEval \cite{chen-et-al21eval-llms-code} and MBPP \cite{austin-et-al21program-synthesis}) in Appendix~\ref{app:code-generation}.  

\noindent\textbf{Dataset and models.} We use the MATH dataset~\citep{hendrycksmath2021}, which contains competition-style high school math problems spanning seven subjects (Algebra, Counting \& Probability, Geometry, Intermediate Algebra, Number Theory, Precalculus, and Prealgebra). We randomly sample 2{,}000 problems for evaluation.
We run experiments with two reasoning models: DeepSeek-R1-Distill-Llama-8B and DeepSeek-R1-Distill-Qwen-7B. 

\noindent\textbf{Pass@$k$ gradient computation.} For computational efficiency, we compute pass@1 gradients with respect to policy parameters in the language model's final hidden layer (dimension $d=4096$ for Llama-8B, $d=3584$ for Qwen-7B). We compute pass@$k$ gradients using Monte Carlo estimates based on \eqref{eq:pass@k-grad} and pass@$k$ estimates developed in prior work \citep{walder-karkhanis25neurips,chen-et-al21eval-llms-code}. 

\noindent\textbf{Setup.} We create filtered data sets $\mathcal{D}_{\delta_1, \delta_2}$ of prompts and responses with varying difficulty thresholds $(\delta_1, \delta_2)$ consisting of 
(i) hard prompts ($p_{\theta}(x_i) < \delta_2$) and (ii) easy prompts ($p_{\theta}(x_i) > \delta_1$). We test different  combinations with $\delta_1 \in \{0.80, 0.85, 0.90\}$ and $\delta_2 \in \{0.05, 0.10, 0.15\}$. For each combination, we compute estimates of the agreement scores~$a_{\theta}(x)$ defined in \eqref{eq:alignment-score}, pass@$k$ weights~$w_{k,\theta}(x)$ (see \eqref{eq:pass@k-weights-formula}) and estimated average weighted agreement scores over prompts, which corresponds to pass@$k$-pass@1 gradient inner product by Proposition~\ref{prop:grad-conflict}.  
See Appendix~\ref{app:threshold_robustness} for all threshold configurations.

\begin{figure}[htbp]
  \centering
  \begin{subfigure}{\textwidth}
    \centering
    \includegraphics[width=\textwidth]{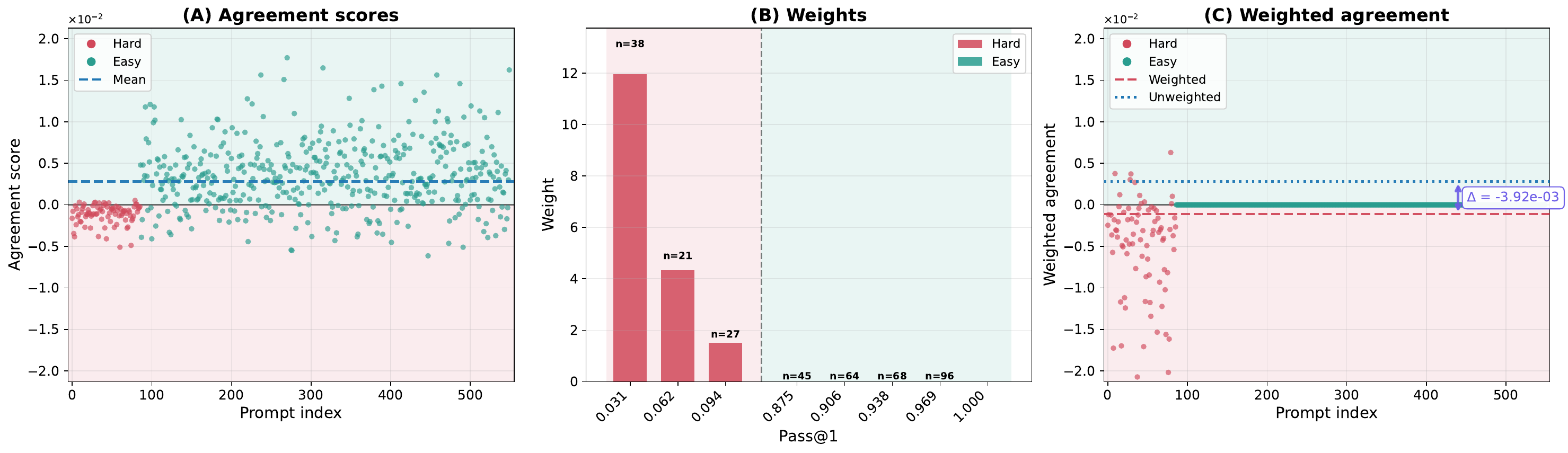}
    \caption{DeepSeek-R1-Distill-Llama-8B: $\delta_1=0.85$, $\delta_2=0.10$ (550 prompts: 86 hard, 464 easy, ratio 5.4:1)}
    \label{fig:llama_085_010}
  \end{subfigure}
  
  \vspace{1em}
  
  \begin{subfigure}{\textwidth}
    \centering
    \includegraphics[width=\textwidth]{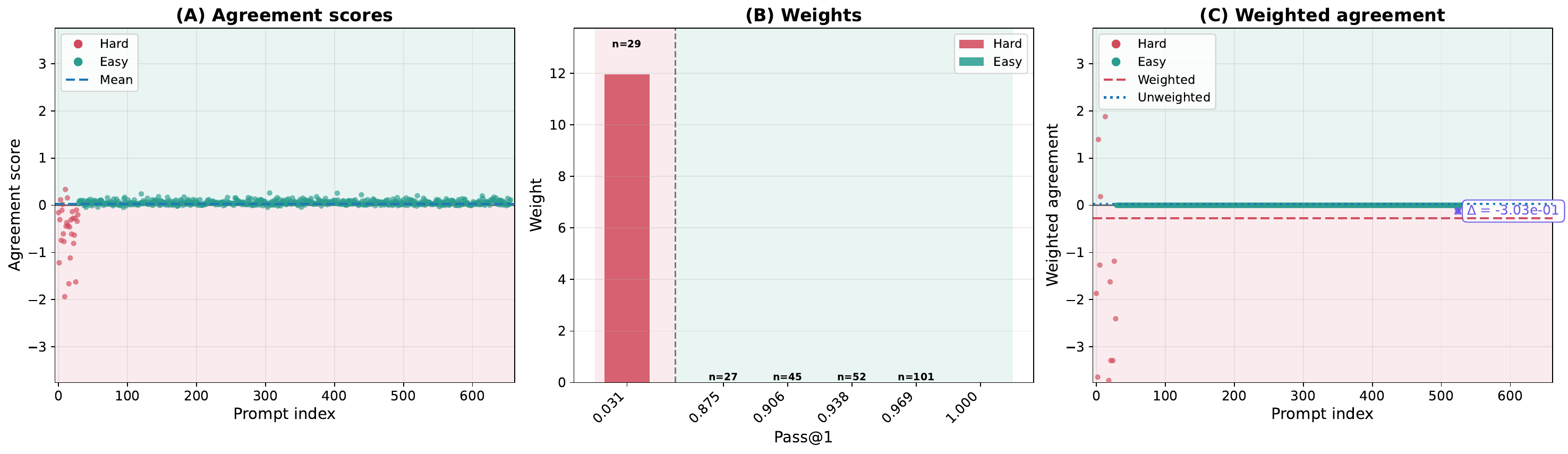}
    \caption{DeepSeek-R1-Distill-Qwen-7B: $\delta_1=0.85$, $\delta_2=0.05$ (656 prompts: 29 hard, 627 easy, ratio 21.6:1)}
    \label{fig:deepseek_085_005}
  \end{subfigure}
  
\caption{
  \textbf{Pass@$k$-induced reweighting produces population-gradient conflict in filtered MATH subsets.}
  Both models demonstrate how pass@$k$ reweighting causes gradient conflict with pass@1, despite different threshold settings.
  Agreement scores are computed relative to the pass@1 gradient of filtered prompts.
  \textbf{Left column (A):} Agreement scores show separation between hard negatives (red) and easy positives (green), with unweighted means shown as blue dashed lines.
  \textbf{Middle column (B):} Pass@$k$ weights grouped by pass@1 probability, showing extreme weight disparity ($\sim$10$^{28}$:1) between hard and easy prompts.
  \textbf{Right column (C):} Weighted contributions (estimates of pass@$k$ vs pass@1 inner products) demonstrate how pass@$k$ reweighting induces negative pass@$k$/pass@1 gradient alignment.
  \emph{Purple arrows highlight the critical downward shift} from the unweighted mean (blue dotted, same as plot A) to the weighted mean (red dashed).
  \textbf{(a) DeepSeek-Llama:} The \emph{downward shift} from $+2.80 \times 10^{-3}$ to $-1.12 \times 10^{-3}$ ($\Delta = -3.92 \times 10^{-3}$) flips the gradient from positive to negative, resulting in inner product $-0.613$. 
  \textbf{(b) DeepSeek-Qwen:} The \emph{dramatic downward shift} from $+2.97 \times 10^{-2}$ to $-2.74 \times 10^{-1}$ ($\Delta = -3.04 \times 10^{-1}$) demonstrates even stronger gradient conflict. Despite only 29 hard prompts vs 627 easy (21.6:1 ratio), the extreme weight disparity causes hard negatives to completely dominate, resulting in inner product $-181$.
}
  \label{fig:gradient_misalignment_two_models}
\end{figure}

\noindent\textbf{Discussion.} Our experiments validate the theoretical predictions in several ways. First, the agreement scores (column A in Figure~\ref{fig:gradient_misalignment_two_models}) show clear separation between hard prompts (red points, negative agreement scores clustered below zero) and easy prompts (green points, positive agreement scores above zero), confirming that prompt interference exists in practice. The unweighted mean agreement (blue dashed line) is positive, indicating that under uniform weighting, the population gradient would improve pass@1. Second, the pass@$k$ weights (column B in Figure~\ref{fig:gradient_misalignment_two_models}) reveal extreme disparity in how pass@$k$ values different prompts. The histogram shows that hard prompts (red bars, low pass@1 values $\lesssim 0.1$) receive weights of $n=38$ at scale 12, $n=21$ at scale 4, and $n=27$ at scale 1, while easy prompts (green bars, high pass@1 values $\gtrsim 0.85$) receive negligible weights near $10^{-28}$. This creates weight ratios exceeding $10^{28}:1$, demonstrating the extreme reweighting mechanism our theory identifies: pass@$k$ dramatically amplifies the influence of low-success prompts while effectively ignoring high-success prompts. Third, column C demonstrates the consequence of this reweighting through weighted agreement scores. The purple arrows in panel C highlight the critical downward shift from the unweighted mean (blue dotted line, matching column A) to the weighted mean (red dashed line). For Llama-8B (Figure~\ref{fig:llama_085_010}), this shift is $\Delta = -3.92\times10^{-3}$, moving from $+2.80\times10^{-3}$ to $-1.12\times10^{-3}$ and flipping the gradient alignment from positive to negative, resulting in an inner product of $-0.613$. For Qwen-7B (Figure~\ref{fig:deepseek_085_005}), despite having only 29 hard prompts versus 627 easy prompts (a 21.6:1 ratio heavily favoring easy prompts), the extreme weight disparity causes an even more dramatic shift of $\Delta = -3.04\times10^{-1}$, moving from $+2.97\times10^{-2}$ to $-2.74\times10^{-1}$, yielding a strongly negative inner product of $-181$.

 \noindent This validates Proposition~\ref{prop:grad-conflict}: when the negative covariance between pass@$k$ weights $w_{k,\theta}(x)$ and agreement scores $a_\theta(x)$ dominates the positive term $\mathbb{E}[w_{k,\theta}(x)] \cdot \|\nabla J_1(\theta)\|^2$, gradient conflict occurs. The visualization supports the proposed mechanism: pass@$k$ upweights hard prompts (column B) $\rightarrow$ these hard prompts exhibit negative interference $a_\theta(x) < 0$ (column A) $\rightarrow$ the reweighted gradient conflicts with pass@1's gradient (column C). The robustness across threshold configurations (see Appendix~\ref{app:threshold_robustness}, Figure~\ref{fig:all_threshold_configs_llama}) demonstrates this is not an artifact of threshold selection. 

\noindent\textbf{Additional pass@$k$ policy optimization experiment.} 
We performed end-to-end pass@$8$ policy optimization on Qwen2.5-Math-1.5B-Instruct using on-policy RL with verifiable rewards. We construct the training pool from the train split of ricdomolm/MATH-500, which contains approximately 12,000 MATH problems after excluding the 500-problem MATH-500 evaluation set. For each of the five independent training runs, we randomly sample 2,000 training problems using the run seed, while evaluation is performed on all 500 problems from the held-out HuggingFaceH4/MATH-500 benchmark. We additionally verify that no normalized problem text is shared between the sampled training pool and the held-out evaluation set. 

We optimize rank-16 LoRA adapters inserted into the \texttt{q\_proj, k\_proj, v\_proj, o\_proj, gate\_proj, up\_proj}
and \texttt{down\_proj} modules of transformer layers, resulting in  2,637,824 trainable parameters. The model is trained for 500 optimization steps using a pass@8 objective with binary correctness rewards. At checkpoints $\{0,125,250,375,500\}$, we evaluate using 32 sampled responses for each of 500 fixed held-out MATH problems. Pass@1 is computed as the mean correctness, while pass@8 is estimated using the standard unbiased combinatorial estimator. To reduce Monte Carlo variability, every checkpoint within a run is evaluated using the same held-out prompts, sampling configuration, and random seed, allowing paired comparisons relative to the step 0 checkpoint.

At every checkpoint, we recompute the gradient interference diagnostic using the same responses to estimate success probabilities and reward-weighted score-function gradients, all with respect to the trainable LoRA parameters. 
We report the results of the experiment in Table~\ref{tab:training-checkpoints}. 

 \vspace{-1.2em}
\begin{table}[ht]
\centering
\caption{\textbf{Pass@$k$ optimization can hurt pass@1 performance}: under pass@$8$ policy optimization, pass@$8$ improves by $+1.65 \pm 0.49$ percentage points, while pass@1 decreases by $2.40 \pm 0.49$ percentage points relative to initialization. Consistent with our theory, the population-gradient cosine becomes negative early during optimization and remains negative as pass@$1$ degrades. 
Mean $\pm 95 \%$ confidence interval reported during training across $5$ independent runs.}
\smallskip 
\label{tab:training-checkpoints}
\begin{tabular}{cccc}
\toprule
Training step & pass@1 & pass@8 & Cosine similarity($\nabla \text{pass@1},\nabla \text{pass@8}$)\\
\midrule
0   & $0.7007 \pm 0.0027$          & $0.8380 \pm 0.0016$          & $+0.032 \pm 0.028$ \\
125 & $0.6979 \pm 0.0023$          & $0.8448 \pm 0.0021$          & $-0.066 \pm 0.039$ \\
250 & $0.6919 \pm 0.0024$          & $0.8501 \pm 0.0019$          & $-0.154 \pm 0.044$ \\
375 & $0.6846 \pm 0.0021$          & $0.8503 \pm 0.0017$          & $-0.133 \pm 0.053$ \\
500 & $\mathbf{0.6767 \pm 0.0039}$ & $\mathbf{0.8545 \pm 0.0025}$ & $\mathbf{-0.246 \pm 0.044}$ \\
\bottomrule
\end{tabular}
\end{table}

\vspace{-1.5em}
\section{Conclusion and Future Work}

We studied why pass@$k$ policy optimization can degrade pass@1 and showed that the conflict arises when pass@$k$ reweighting amplifies negatively interfering prompts. Central to our analysis is the concept of \emph{prompt interference}: because a language model shares parameters across a heterogeneous prompt distribution, improving some prompts can adversely affect others. 
More broadly, prompt interference provides a useful lens for post-training objectives that redistribute optimization effort across prompts, suggesting practical diagnostics and interference-aware mitigation strategies for objective-induced trade-offs. It also suggests a path toward more effective post-training methods that exploit the structure of the prompt distribution rather than treating prompts independently.

\begin{ack}
We thank Aviral Kumar for pointing us to related work on negative interference between easy and hard problems, pass@$k$ policy optimization \cite{qu-et-al25blog-pope,qu-et-al26pope}, and prior work on ray interference in deep reinforcement learning \cite{schaul-et-al19ray-interference}. 

We used ChatGPT during the development of this project to assist with coding, the development of the toy example, and brainstorming and refining theoretical ideas. The manuscript, including all mathematical claims and proofs, was written and verified by the authors, with ChatGPT used for suggestions on selected technical points, wording, and presentation. The authors take full responsibility for the content of this manuscript.
\end{ack}

\bibliographystyle{plain}
\bibliography{references}

%%%%%%%%%%%%%%%%%%%%%%%%%%%%%%%%%%%%%%%%%%%%%%%%%%%%%%%%%%%%

\newpage
\appendix

\tableofcontents

\newpage

\section{Related Work}
\label{sec:related-work}

\noindent\textbf{Inference-time fine-tuning.} Inference-time objective optimization has recently witnessed a surge of interest for optimizing language model performance for specific inference-time procedures such as win rate \citep{balashankar-et-al25infalign}, Best-of-N \citep{sessa-et-al25bond,amini-et-al25vBoN,chow-et-al25infaware-BoN,beirami-et-al25theoretical-BoN,aminian-et-al25bon}, majority voting \citep{uesato-et-al22,lightman-et-al24,du-et-al25} or more recently their combination \citep{di-et-al25best-of-majority}. Davis and Recht \cite{davis-recht25} showed that several existing algorithms for post-training large language models (such as rejection sampling fine-tuning (see, e.g., \cite{xiong-et-al25minimalist}) and GRPO \citep{shao-et-al24deepseekmath}) can be seen as instances of stochastic gradient ascent for maximizing a monotonically increasing function of the probability of giving a correct answer given a prompt, in the binary reward setting. 

\noindent\textbf{Pass@$k$ optimization.} In this work, we focus on the special case of the pass@$k$ metric which has been introduced as an evaluation metric in \cite{chen-et-al21eval-llms-code} for code generation tasks, similarly to a metric proposed earlier in \cite{kulal-et-al19spoc}. Given access to verifiable rewards,  Chen et al. \cite{chen-et-al21eval-llms-code} proposed an unbiased estimator of the pass@$k$ metric. Inspired by inference-aware optimization gains, Tang et al.~\cite{tang-et-al25icml} recently proposed to directly perform pass@$k$ policy optimization with variance-reduced pass@$k$ policy gradients using the control-variates method.  
Concurrently, Walder and Karkhanis~\cite{walder-karkhanis25neurips} introduced new lower variance unbiased estimators for pass@$k$ and its gradient using more general U-statistics and demonstrated its performance for solving hard problems.  

\noindent Mahdavi et al.~\cite{mahdavi-et-al25pg-passk} derived expressions for the pass@$k$ objective and pass@$k$ policy gradients as presented in \eqref{eq:pass@k-grad} and proposed a new pass@$k$ gradient reweighting method with a tunable parameter~$\gamma$ different from $k$. Yu \cite{yu25} also derives per-prompt pass@$k$ policy gradients and shows that per-prompt pass@$k$ and pass@1 gradients are positively collinear vectors, questioning the soundness of considering the pass@$k$ objective and advocating for different exploration mechanisms while keeping pass@$k$ as an inference-time evaluation metric. We show in our work that the \textit{expected} pass@$k$ and pass@1 gradients over the prompt distribution may not necessarily be positively collinear.  

\noindent  Chen et al.~\cite{chen-et-al25passktraining-exploration} investigated how pass@$k$ model training balances the exploration and exploitation capabilities of large language models. Jiang et al.~\cite{jiang-et-al25risk-sensitive-rl} further considered a different risk-sensitive objective than pass@$k$ based on a soft maximum (log-sum-exp) objective depending on a temperature parameter, providing a knob to trade-off between pass@$k$ and pass@1 performance. Peng et al.~\cite{peng-et-al25simko} proposed Simple Pass@$K$ Optimization (SimKO) which redistributes gradient updates across the top-$K$ token-level probability candidates to promote exploration and improve the balance between exploitation and exploration.

\noindent Thrampoulidis et al.~\cite{thrampoulidis-et-al25} connect policy gradient algorithms for pass@$k$ optimization \citep{mahdavi-et-al25pg-passk,tang-et-al25icml,walder-karkhanis25neurips} to advantage shaped GRPO methods \citep{chen-et-al25passktraining-exploration,shao-et-al24deepseekmath,deng-et-al25,mroueh25rlvr-grpo}.

Kazdan et al.~\cite{kazdan-et-al25passk} show how pass@$k$ can be estimated even in cases where the number of samples per problem is smaller than the number of attempts $k$ and use dynamic sampling to allocate sampling budget according to prompt difficulty. 

\noindent To the best of our knowledge, prior work on pass@$k$ optimization does not theoretically characterize when and why optimizing pass@$k$ can degrade pass@1 under pass@$k$ policy optimization; we provide such a characterization.

\noindent\textbf{Interference.} The concept of \textit{ray interference} has been previously introduced and investigated in multi-task reinforcement learning \citep{schaul-et-al19ray-interference}. As an analogy with multi-task learning, tasks can be viewed as prompts (or problems to solve in the special case of mathematical reasoning) in our LLM post-training setting. Instead of composite optimization objectives in the form of a sum over tasks, pass@$k$ objectives are expected probabilities of success over a prompt distribution, which may be discrete or continuous (see \eqref{eq:pass@k}). Our definition of prompt interference (Definition~\ref{def:prompt-interference}) parallels \cite[Definition~1]{schaul-et-al19ray-interference}, using gradients of pass@1 in our LLM post-training setting rather than gradients of task loss functions in multi-task learning. Related notions of gradient conflict/interference have also been studied in multi-task learning (e.g., \cite{yu-et-al20gradient-surgery-mtl}). To extend the multi-task learning analogy discussed in \cite{qu-et-al26pope}, tasks can also be thought of as clusters of prompts labeled by degree of difficulty. 

\noindent Inspired by \cite{schaul-et-al19ray-interference}, Qu et al.~\cite{qu-et-al26pope} recently investigated learning to reason on hard problems when no correct rollouts are sampled and argue that pass@$k$ optimization does not address this challenge. They emphasize the role of ray interference in hindering learning on heterogeneous prompt mixtures; in particular, Section~3.2 discusses negative interference between easy and hard problems, and how pass@$k$ policy optimization does not mitigate its consequences. Conceptually, Qu et al.~\cite{qu-et-al26pope} study interference as an exploration/solvability bottleneck on heterogeneous mixtures, while we analyze the pass@$k$ objective itself and show how prompt-level gradient conflict can induce an explicit pass@$k$--pass@1 trade-off.

\noindent Our findings are consistent with the empirical insights and intuitions in \cite[Section~3.2]{qu-et-al26pope}. We offer a complementary theoretical view by formalizing prompt interference for LLM post-training and providing conditions under which pass@1 degradation can occur under pass@$k$ policy optimization via prompt-level gradient conflict (see Proposition~\ref{prop:grad-conflict}, Corollary~\ref{cor:dominating-neg-interfer-prompts}, and Proposition~\ref{prop:pass@1-degradation}). Our results also connect to several open questions highlighted by \cite{qu-et-al26pope}.

\noindent While our paper focuses on the theoretical understanding of the pass@$k$ vs pass@1 trade-off and limitations of pass@$k$ policy optimization, Qu et al.~\cite{qu-et-al26pope} propose a framework (Privileged On-Policy Exploration, POPE) to mitigate ray interference and improve the solvability of hard problems by guiding on-policy exploration. Our theoretical insights may suggest alternative mitigation mechanisms and clarify when such interventions are necessary and effective. 

\section{Limitations}
\label{sec:limitations}

In this section, we discuss limitations of our theoretical and empirical results. 

First, the gradient conflict characterization in Proposition~\ref{prop:grad-conflict} is local to a fixed policy and does not imply monotonic degradation of pass@1 throughout stochastic, multi-step training. Both the pass@$k$ weights and prompt agreement scores evolve with the policy, so the gradient alignment may change over training and pass@1 may eventually stabilize or recover. The success-probability separation assumption is used only to obtain the explicit threshold-in-$k$ sufficient condition in Proposition~\ref{prop:phase-transition-in-k}; when success probabilities of negatively interfering and remaining prompts overlap, the exact characterization in Proposition~\ref{prop:grad-conflict} remains valid, but the threshold guarantee need not hold. Similarly, the smoothness assumption is required for the finite-step guarantee in Proposition~\ref{prop:pass@1-degradation} rather than for the core population-gradient identity.

Second, our original MATH experiments use hard/easy filtered subsets and a reduced-space gradient diagnostic as controlled stress tests. The additional end-to-end experiment at the end of Section~\ref{sec:experiments} removes these restrictions by evaluating all 500 held-out MATH-500 prompts and computing gradient diagnostics with respect to the trainable LoRA parameters. 
Nevertheless, the estimates use 32 sampled responses per prompt and may therefore have substantial per-prompt Monte Carlo variance, especially for low-success prompts.

Third, although pass@1 degradation under pass@$k$ optimization has been observed previously~\citep{tang-et-al25icml,walder-karkhanis25neurips}, it would be interesting to conduct more simulations beyond our illustrative example at the end of Section~\ref{sec:experiments}, which is limited to one 1.5B model, one mathematical-reasoning benchmark and $k=8$. Broader empirical evaluation would allow to further explore how frequently and strongly the phenomenon occurs across model families, scales, tasks and values of $k$.  

Finally, our results suggest several possible mitigation strategies, including clipping pass@$k$ weights, agreement-aware downweighting of negatively interfering prompts, and first-order pass@1 gradient guards inspired by the multi-task gradient-conflict literature (e.g., \cite{yu-et-al20gradient-surgery-mtl}). Testing and comparing these different strategies in full LLM post-training is an interesting avenue for future work.  

\section{Toy example}  
\label{sec:toy-example}

\subsection{Setting and illustration of negative prompt interference}

We consider a minimal contextual-bandit abstraction of verifier-based multi-sample LLM evaluation as a running example. 
A prompt is a context $x \in \mathcal{X}$ drawn from a mixture distribution with a latent difficulty label $\ell(x)\in\{\textnormal{easy},\textnormal{hard}\}$ with a fixed proportion of hard labels: $\mathbb{P}(\ell(x)=\textnormal{hard})=p_{\mathrm{hard}}.$ Each prompt is represented by a 2D feature vector $\psi(x) = [1, s(x)]^{\top} \in \mathbb{R}^2$ with a bias $1$ and a scalar feature $s(x) \in \mathbb{R}.$ 
Conditioned on difficulty, the scalar feature follows overlapping Gaussians whose separation is controlled by a parameter $\delta > 0$: if $\ell(x)=\textnormal{easy}$ then $s(x)\sim\mathcal{N}\!\left(-\delta/2,\,1\right),$ and if $\ell(x)=\textnormal{hard}$ then $s(x)\sim\mathcal{N}\!\left(+\delta/2,\,1\right).$ Smaller $\delta$ yields greater overlap between easy and hard feature distributions. A completion is represented as a discrete action $y \in\{0,1\}.$ Each prompt has a unique correct completion determined by its difficulty: $y^{\star}(x) = 0$ if $\ell(x)=\text{easy}$ and $y^{\star}(x) = 1$ if $\ell(x)=\text{hard}$.
The verifier reward is deterministic binary: $r(x,y)=\mathbf{1}\{y=y^{\star}(x)\}\,.$ 
We use a two-action softmax policy parameterized by a single vector $\theta\in\mathbb{R}^2$ such that $\pi_\theta(1| x)=\sigma(\theta^\top \psi(x))$ and $\pi_\theta(0|x)=1-\sigma(\theta^\top \psi(x))$, where
$\sigma(t)=\frac{1}{1+\exp(-t)}$ is the sigmoid function. The policy shares parameters across heterogeneous subpopulations (easy vs hard). For later use, we also note that we initialize the policy at a reference policy $\theta_{\mathrm{ref}}$ that is globally biased toward action $y = 0$, inducing high initial success on easy prompts and low initial success on hard prompts. 

\begin{wrapfigure}{r}{0.4\linewidth}
  \vspace{-1\baselineskip}
  \centering
  \includegraphics[width=0.75\linewidth]{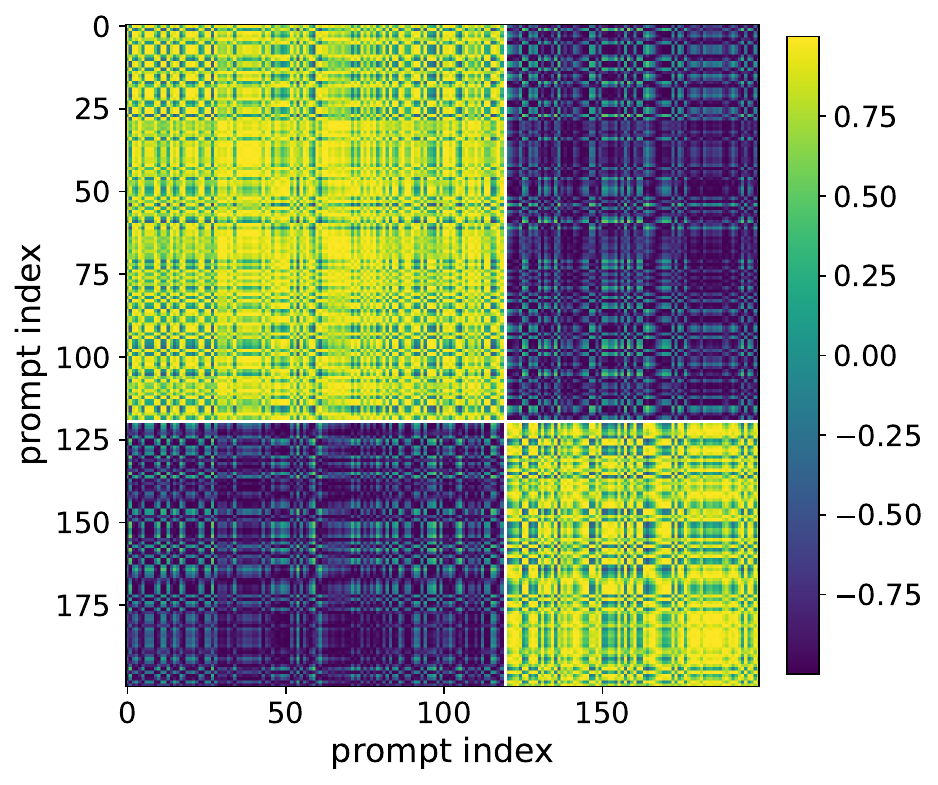}
  \caption{Cosine kernel heatmap: $\cos(\nabla p_{\theta}(x), \nabla p_{\theta}(x'))$ for subsamples of prompts: 120 easy and 80 hard among a total of 6000 samples. Blue regions show negative prompt interference.}
  \label{fig:cosine-kernel-heatmap}
\end{wrapfigure}

\noindent\textbf{Illustration of negative prompt interference.} 
It follows that the per-prompt success probability defined in \eqref{eq:ptheta-def} is given by $p_\theta(x) = \pi_\theta(y^{\star}(x)|x)$ for any prompt $x \in \mathcal{X}.$ A simple computation yields: 
\begin{equation}
\label{eq:grad-pass@1-toy}
\nabla p_{\theta}(x)= \begin{cases}
- z_{\theta}(x)\,\psi(x), &\text{if } \ell(x)=\text{easy},\\
+\,z_{\theta}(x)\,\psi(x), &\text{if } \ell(x)=\text{hard}\,,
\end{cases}
\end{equation}
where $z_{\theta}(x) := \sigma(\theta^\top \psi(x)) (1- \sigma(\theta^\top \psi(x))) \geq 0\,.$  As a consequence, two prompts that have a similar representation will have opposite per-prompt pass@1 gradients and will hence be negatively interfering. This will be the case in our example for prompts in the overlap between the two Gaussians: two such prompts can have different labels (\text{hard} and \text{easy}) while having a similar prompt representation (with a similar scalar feature $s(x)$). See Figure~\ref{fig:cosine-kernel-heatmap} for an illustration. 

\noindent In the next section, we show in our toy example how negative prompt interference can lead to pass@1 decrease under one-step pass@$k$ policy gradient ascent. 

\subsection{Tradeoff: pass@$k$ ascent can decrease pass@1}

Using the toy example introduced in the previous section, we show in this section that pass@1 can decrease under one-step pass@$k$ policy gradient ascent.  

\noindent Fix $\delta=0.2$ and consider two realized prompts in the overlap region between the 2 Gaussians:
a prompt $x_e$ with $\ell(x_e)=\text{easy}$ and $\psi(x_e)=[1,-0.1]^\top$, and $x_h$ with $\ell(x_h)=\text{hard}$ and $\psi(x_h)=[1,+0.1]^\top$. These two prompts
have nearly identical representations as $\cos(\psi(x_e),\psi(x_h))\simeq 0.98$ is close to 1. Because of \eqref{eq:grad-pass@1-toy}, they are therefore negative interfering as $\kappa_{\theta}(x_e, x_h) = \langle \nabla p_\theta(x_e), \nabla p_\theta(x_h) \rangle = -z_\theta(x_e)z_\theta(x_h)\langle \psi(x_e),\psi(x_h)\rangle<0$ since $\langle \psi(x_e),\psi(x_h)\rangle>0.$ Numerically we have $\kappa_{\theta}(x_e, x_h) \simeq - 0.01$ and $\cos( \nabla p_\theta(x_e), \nabla p_\theta(x_h)) \simeq - 0.98$ which means that the per-prompt pass@1 gradients are almost antiparallel. On this two-point empirical distribution, the pass@1 and pass@$k$ gradients are given by: 
\begin{equation*}
\nabla J_1(\theta) = \frac{1}{2}(\nabla p_\theta(x_e) + \nabla p_\theta(x_h))\,,\quad
\nabla J_{10}(\theta) = \frac{1}{2}(w_e \nabla p_\theta(x_e) + w_h \nabla p_\theta(x_h))\,,
\end{equation*}
where the pass@$10$ weight $w_e := w_{10,\theta}(x_e) \simeq 2.33 \times10^{-7}$ whereas $w_h := w_{10,\theta}(x_h) \simeq 3.88.$ Pass@$10$ upweights the hard prompt ($p_{\theta}(x_h) \simeq 0.10$ low) and aggressively downweights the easy prompt $x_e$ (large $p_{\theta}(x_h) \simeq 0.86$). Indeed, recall that $w_{k,\theta}(x) = k (1- p_{\theta}(x))^{k-1}$ which is near $k$ when $p_{\theta}(x) \approx 0$ but near $0$ when $p_{\theta}(x) \approx 1\,.$ 
As a consequence, the pass@$10$ gradient is therefore dominated by $\nabla p_{\theta}(x_h)$ which conflicts with $\nabla p_{\theta}(x_e)$. Importantly, we have $\cos(\nabla J_1(\theta), \nabla J_{10}(\theta)) = -0.77 < 0$ meaning that the pass@1 and pass@$10$ are conflicting and strongly anti-aligned. With one gradient-ascent step on pass@$10$ $\theta^+=\theta+\eta \nabla J_{10}(\theta)$ with $\eta=5$, we have: 
\begin{equation*}
J_1(\theta)  \simeq 0.48 \;\to\; J_1(\theta^+) \simeq 0.46 \quad (\downarrow),\qquad 
J_{10}(\theta) \simeq 0.83 \;\to\; J_{10}(\theta^+)=0.95 \quad (\uparrow). \nonumber
\end{equation*}
This shows that pass@1 decreases while pass@$k$ increases under pass@$k$ optimization. 
In our toy example, we have observed that negative prompt interference can induce a conflict between pass@1 and pass@$k$ gradients when it is amplified by the induced pass@$k$ prompt reweighting. 

\begin{figure}[h]
\centering
  \includegraphics[width=0.7\linewidth]{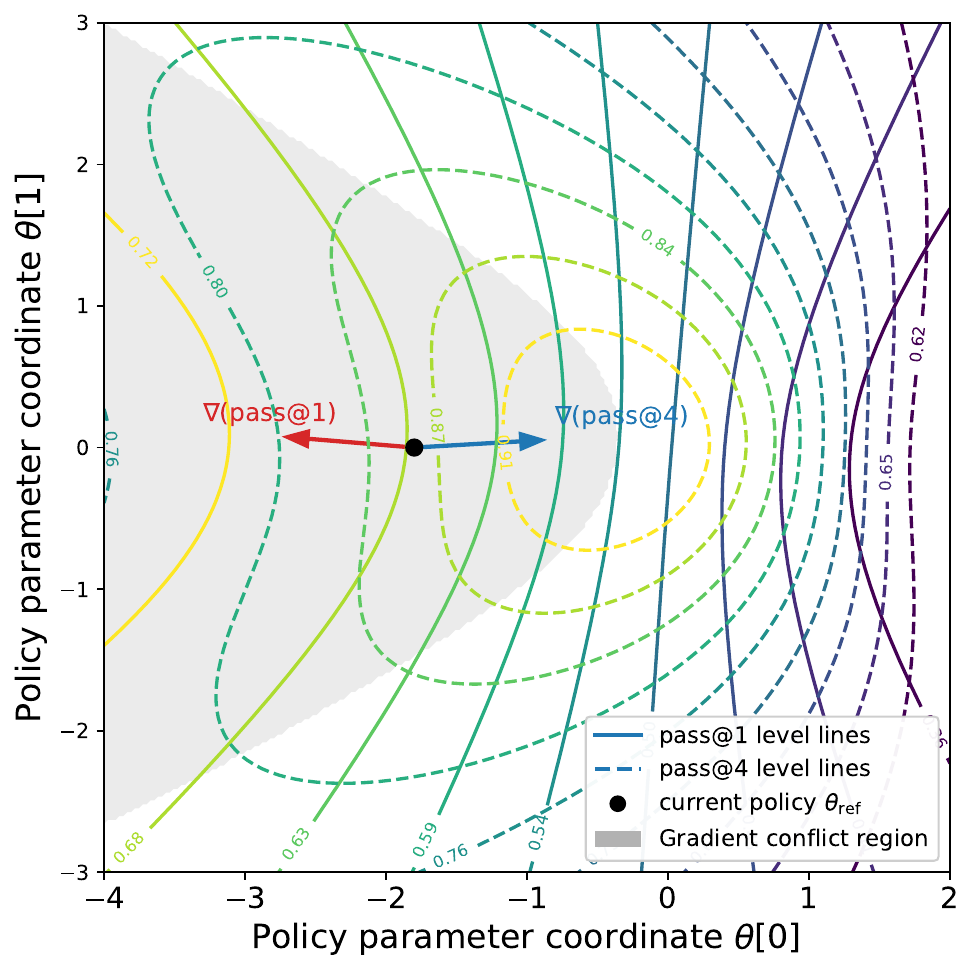}
  \caption{
    Contour plots of pass@1 and pass@$k$ objectives in the policy parameter space. 
    Gradients of pass@$k$ and pass@1 with respect to policy parameters are conflicting in the gray area. 
  }
  \label{fig:front-page}
\end{figure}

\newpage
\section{Discussion of the covariance condition in Proposition~\ref{prop:grad-conflict}}
\label{appx:covar-condition}

Introducing the correlation $\rho_{w,a}(\theta) := \frac{\cov(w_{k,\theta}(x),a_{\theta}(x))}{\sigma_{w}(\theta) \cdot \sigma_{a}(\theta)}$ where $\sigma_{w}(\theta), \sigma_{a}(\theta)$ are respectively the standard deviations of $w_{k,\theta}(x)$ and $a_{\theta}(x)$ and supposing $\sigma_{w}(\theta) \sigma_{a}(\theta) > 0$, we can then rewrite condition \eqref{eq:conflict-1} equivalently as: 
\begin{equation}
\rho_{w,a}(\theta) < - \frac{\mathbb{E}[w_{k,\theta}(x)] \cdot \|\nabla J_1(\theta)\|^2}{\sigma_{w}(\theta) \cdot \sigma_{a}(\theta)}\,.
\end{equation}
This condition emphasizes that conflict requires strong negative correlation between hardness weights and agreement scores as well as sufficient variance. If the weights barely vary ($\sigma_w(\theta) \simeq 0$), pass@$k$ cannot strongly shift focus across prompts and conflict is harder. If agreement scores barely vary ($\sigma_a(\theta) \simeq 0$) there is no meaningful conflict structure to exploit that can lead to misaligned gradients.    

\section{A kernel view of gradient conflict beyond $k=1$}
\label{sec:kernel-grad-conflict-k-m}

In this section, we further investigate the gradient conflict phenomenon between pass@$k$ and pass@$m$ for any $k, m \geq 1$. 
Using the kernel~$\kappa_{\theta}$ defined in \eqref{eq:kernel-def-1}, we now relate pass@$k$ and pass@$m$ gradients for any $k, m \geq 1$ using \eqref{eq:pass@k-grad} to obtain: 
\begin{equation}
\label{eq:inner-grad-k-grad-m}
\langle \nabla J_k(\theta), \nabla J_m(\theta) \rangle 
= \mathbb{E}_{x, x' \sim \mathcal{D}}[w_{k, \theta}(x) w_{m, \theta}(x') \kappa_{\theta}(x,x')]\,.
\end{equation}

This identity separates two factors: (1) where each objective focuses via the weights $w_{k,\theta}(x), w_{m,\theta}(x')$ with prompt reweighting toward `hard' prompts as $k$ grows; and (2) whether those regions interfere via the sign structure of $\kappa_{\theta}(x,x')$ which is independent of the parameter $k$. Conflict between pass@$k$ and pass@$m$ gradients happens when the objectives emphasize prompt regions whose cross-similarity (encoded by the kernel) is negative.

The following result which is an immediate consequence of \eqref{eq:inner-grad-k-grad-m} shows that conflict requires interference. 

\begin{proposition}[No interference regime: No negative transfer $\implies$ no conflict] 
For any policy parameter $\theta \in \mathbb{R}^d$, if $\kappa_{\theta}(x,x') \geq 0$ for $\mathcal{D} \times \mathcal{D}$ a.e. $(x,x')$ then, $\langle \nabla J_k(\theta), \nabla J_m(\theta) \rangle \geq 0$ for all $k, m \geq 1\,.$
\end{proposition}

If pairwise prompt similarities are nonnegative on the support of the prompt distribution~$\mathcal{D}$ then all pass@$k$ objectives are mutually aligned. Therefore, conflict can only occur when pairwise similarities are not always nonnegative and depends on how the weights $w_{k, \theta}, w_{m, \theta}$ shift mass. In other words, pass@$k$ objectives cannot conflict unless there is negative transfer between some prompts. 

\section{Proofs} 

\subsection{Proof of Proposition~\ref{prop:grad-conflict}: Gradients conflict characterization}

\begin{proof}
First recall from \eqref{eq:pass@k-grad} that for any $k \geq 1$ and any $\theta \in \mathbb{R}^d$,
\begin{equation}
\nabla J_k(\theta) = \mathbb{E}_{x \sim \mathcal{D}}[ w_{k}(p_{\theta}(x)) \nabla p_{\theta}(x)] = \mathbb{E}_{x \sim \mathcal{D}}[ w_{k}(p_{\theta}(x)) \nabla J_1(x;\theta)]\,.
\end{equation}
Then it immediately follows that: 
\begin{align}
\label{eq:inner-prod-proof}
\langle \nabla J_k(\theta), \nabla J_1(\theta) \rangle 
&= \langle \mathbb{E}_{x \sim \mathcal{D}}[ w_{k}(p_{\theta}(x)) \nabla J_1(x;\theta)],  \nabla J_1(\theta) \rangle  \nonumber\\
&= \mathbb{E}_{x \sim \mathcal{D}}[\langle w_{k}(p_{\theta}(x)) \nabla J_1(x;\theta),  \nabla J_1(\theta) \rangle]  \nonumber\\
&= \mathbb{E}_{x \sim \mathcal{D}}[w_{k}(p_{\theta}(x)) a_{\theta}(x)]\,, 
\end{align}
where we recall that $a_{\theta}(x) := \langle \nabla J_1(x;\theta), \nabla J_1(\theta) \rangle$\,.
By definition of the covariance between two real-valued random variables: 
\begin{align}
\cov(w_{k,\theta}(x), a_{\theta}(x)) &= \mathbb{E}_{x \sim \mathcal{D}}[w_{k,\theta}(x) \cdot a_{\theta}(x)] - \mathbb{E}_{x \sim \mathcal{D}}[w_{k,\theta}(x)] \cdot \mathbb{E}_{x \sim \mathcal{D}}[ a_{\theta}(x)]\\
&= \langle \nabla J_k(\theta), \nabla J_1(\theta) \rangle  - \mathbb{E}_{x \sim \mathcal{D}}[w_{k,\theta}(x)] \cdot \|\nabla J_1(\theta)\|^2\,,
\end{align}
where the second equality follows from using \eqref{eq:inner-prod-proof} and noticing that $\mathbb{E}_{x \sim \mathcal{D}}[ a_{\theta}(x)] =  \|\nabla J_1(\theta)\|^2$ by definition of $a_{\theta}(x)\,.$ 
Rearranging the above identity then gives: 
\begin{equation}
\label{eq:covar-inner-prod}
\langle \nabla J_k(\theta), \nabla J_1(\theta) \rangle = \mathbb{E}_{x \sim \mathcal{D}}[w_{k,\theta}(x)] \cdot \|\nabla J_1(\theta)\|^2 + \cov(w_{k,\theta}(x), a_{\theta}(x))\,.
\end{equation}
As a consequence of \eqref{eq:inner-prod-proof} and \eqref{eq:covar-inner-prod}, we have that $\langle \nabla J_k(\theta), \nabla J_1(\theta) \rangle < 0$ if and only if: 
$\mathbb{E}_{x \sim \mathcal{D}}[w_{k}(p_{\theta}(x)) a_{\theta}(x)] < 0$ which is also equivalent to the condition: 
\begin{equation}
\cov(w_{k,\theta}(x), a_{\theta}(x)) < - \mathbb{E}[w_{k,\theta}(x)] \cdot \|\nabla J_1(\theta)\|^2\,. 
\end{equation}
\end{proof}
\subsection{Proof of Corollary~\ref{cor:dominating-neg-interfer-prompts}: Dominating negatively interfering prompts}

\begin{proof}
First, Proposition~\ref{prop:grad-conflict} yields: 
\begin{equation}
\bigl\langle \nabla J_k(\theta), \nabla J_1(\theta)\bigr\rangle
=
\mathbb{E}\!\left[\omega_{k,\theta}(x)\,a_\theta(x)\right]\,.
\end{equation}
We decompose this inner product into 2 regions: the $m$-strongly negatively interfering region~$\mathcal{X}_{-}(\theta,m)$ and its complement, 
\begin{align}
\label{eq:inner-prod-decomp}
\bigl\langle \nabla J_k(\theta),\ \nabla J_1(\theta)\bigr\rangle
&=
\mathbb{E}\!\left[\omega_{k,\theta}(x)\,a_\theta(x)\right] \nonumber\\
&=
\mathbb{E}\!\left[\omega_{k,\theta}(x)\,a_\theta(x)\,\mathds{1}_{\{x\in\mathcal{X}_{-}(\theta,m)\}}\right]
+
\mathbb{E}\!\left[\omega_{k,\theta}(x)\,a_\theta(x)\,\mathds{1}_{\{x\notin\mathcal{X}_{-}(\theta,m)\}}\right]\,.
\end{align}

We now upperbound each one of the above terms. For the first term, by definition of the set of prompts $\mathcal{X}_{-}(\theta,m)$, we have $a_\theta(x)\le -m$ for any $x \in \mathcal{X}_{-}(\theta,m)\,.$ As for the second term, we show that for any $x \notin \mathcal{X}_{-}(\theta,m),$ $|a_{\theta}(x)| \leq G^2$ using Assumption~\ref{as:E-LS} as follows. 

\noindent\textbf{Bounding $a_{\theta}(x).$} Using the definition of the agreement score $a_{\theta}(x)$ (see \eqref{eq:alignment-score}) together with Cauchy-Schwarz inequality, we have 
\begin{equation}
\label{eq:proof-bound-atheta}
|a_{\theta}(x)| = |\langle \nabla J_1(x;\theta), \nabla J_1(\theta) \rangle| \leq \|\nabla J_1(x;\theta)\| \cdot \|\mathbb{E}[\nabla J_1(x';\theta)]\| \leq \|\nabla J_1(x;\theta)\| \cdot \mathbb{E}[\|\nabla J_1(x';\theta)\|] \,.
\end{equation}
Then recalling the policy gradient expression of $\nabla J_1(x;\theta)$ in \eqref{eq:pass@1pg} and using Jensen's inequality, we have: 
\begin{equation}
\label{eq:pass@1-grad-bound}
\|\nabla J_1(x;\theta)\|^2 = \|\mathbb{E}_{y \sim \pi_{\theta}(\cdot|x)}[r(x,y) s_{\theta}(x,y)]\|^2 \leq \mathbb{E}_{y \sim \pi_{\theta}(\cdot|x)}[\|s_{\theta}(x,y)\|^2] \leq G^2\,,
\end{equation}
where the last inequality follows from using Assumption~\ref{as:E-LS}. Using this inequality in \eqref{eq:proof-bound-atheta} yields the desired inequality: 
\begin{equation}
\label{eq:bound-agreement-score}
\forall x \in \mathcal{X}, \quad |a_{\theta}(x)| \leq G^2\,.
\end{equation}
Back to \eqref{eq:inner-prod-decomp} and using the inequalities established above, we conclude the proof: 
\begin{equation}
\bigl\langle \nabla J_k(\theta),\ \nabla J_1(\theta)\bigr\rangle \leq  - \delta(\theta)\,, 
\end{equation}
where $\delta(\theta) := m W_{-}(k,\theta) - G^2 W_{+}(k,\theta)$ and the aggregated weights~$W_{-}(k,\theta), W_{+}(k,\theta)$ are defined in \eqref{eq:aggregated-weights}. If in addition $\delta(\theta) > 0$, then $\langle \nabla J_k(\theta), \nabla J_1(\theta)\rangle \leq - \delta(\theta) < 0\,.$
\end{proof}

\subsection{Proof of Proposition~\ref{prop:phase-transition-in-k}: Influence of $k$}

\begin{proof}
Recall from Proposition~\ref{prop:grad-conflict} that: 
\begin{equation}
\bigl\langle \nabla J_k(\theta), \nabla J_1(\theta)\bigr\rangle
=
\mathbb{E}\!\left[\omega_{k,\theta}(x)\,a_\theta(x)\right]\,.
\end{equation}
Using again the same decomposition as in \eqref{eq:inner-prod-decomp}, we have: 
\begin{align}
\label{eq:inner-prod-decomp2}
\bigl\langle \nabla J_k(\theta),\ \nabla J_1(\theta)\bigr\rangle
&=
\mathbb{E}\!\left[\omega_{k,\theta}(x)\,a_\theta(x)\right]\nonumber\\
&=
\mathbb{E}\!\left[\omega_{k,\theta}(x)\,a_\theta(x)\,\mathds{1}_{\{x\in\mathcal{X}_{-}(\theta,m)\}}\right]
+
\mathbb{E}\!\left[\omega_{k,\theta}(x)\,a_\theta(x)\,\mathds{1}_{\{x\notin\mathcal{X}_{-}(\theta,m)\}}\right]\,.
\end{align}
Depending on the prompt region, we have two distinct ways to bound the pass@$k$ weights:
\begin{enumerate}[label=(\roman*), leftmargin=0pt, itemindent=*, align=left]
\item For $x\in \mathcal{X}(m,\theta)$, we have by assumption, $p_\theta(x)\le \varepsilon$, hence $\omega_{k,\theta}(x)\ \ge\ k(1-\varepsilon)^{k-1}.$
\item For $x\notin \mathcal{X}(m,\theta)$ we have by assumption, $p_\theta(x)\ge \delta$, hence $\omega_{k,\theta}(x)\ \le\ k(1-\delta)^{k-1}.$
\end{enumerate}
Observing in addition to both items above that $a_\theta(x)\le -m$ for any $x \in \mathcal{X}_{-}(\theta,m)$ and $a_\theta(x)\le G^2$ for any $x \in \mathcal{X}$ (as shown in \eqref{eq:bound-agreement-score} under Assumption~\ref{as:E-LS}) and hence for any $x \notin \mathcal{X}_{-}(\theta,m)$ in particular,  we obtain from \eqref{eq:inner-prod-decomp2} the following inequality: 
\begin{equation}
\bigl\langle \nabla J_k(\theta), \nabla J_1(\theta)\bigr\rangle
\ \le\
-\,k(1-\varepsilon)^{k-1}\,m\,q
\ +\
k(1-\delta)^{k-1}\,G^2\,(1-q)\,,
\end{equation}
where we also use the assumption~$q_{\theta} \geq q\,.$
Thus it suffices to have
\[
k(1-\delta)^{k-1}G^2(1-q)\ <\ k(1-\varepsilon)^{k-1}m q,
\]
to ensure that $\bigl\langle \nabla J_k(\theta), \nabla J_1(\theta)\bigr\rangle < 0\,.$ The above condition is equivalent to: 
\[
\frac{G^2(1-q)}{m q}\ <\ \left(\frac{1-\varepsilon}{1-\delta}\right)^{k-1}.
\]
Taking $\log$ in the above inequality, recalling that $1-\varepsilon > 1-\delta$ and reordering the inequality yields: 
\begin{equation}
k\ > k^{\star} := \ 1\ +\ \frac{\log\!\Bigl(\dfrac{(1-q)\,G^{2}}{q\,m}\Bigr)}
{\log\!\Bigl(\dfrac{1-\varepsilon}{1-\delta}\Bigr)}\,, 
\end{equation}
which concludes the proof. 
\end{proof}

\subsection{Lemma~\ref{lem:smoothness-pass@k}: Smoothness of pass@$k$}
\label{appx:lemma-smoothness}

\begin{lemma}[Smoothness of pass@$k$]
\label{lem:smoothness-pass@k}
Under Assumption~\ref{as:E-LS}, for all $k \geq 1$, 
the pass@$k$ objective $J_k$ is $L_k$-smooth with $L_k := k^2 G^2 + kF,$ i.e. for all $\theta, \theta' \in \mathbb{R}^d,$
\begin{equation*}
|J_k(\theta') - J_k(\theta) - \langle \nabla J_k(\theta), \theta' - \theta \rangle| \leq \frac{L_k}{2}\|\theta' - \theta\|_2^2\,.
\end{equation*}
\end{lemma}

\begin{proof}
To show smoothness of the pass@$k$ objective $J_k$, we bound its Hessian $\nabla^2 J_k(\theta)$ uniformly in the policy parameter~$\theta$ using Assumption~\ref{as:E-LS}. 

Since
$\nabla J_k(\theta) = \mathbb{E}_{x \sim \mathcal{D}}[ w_{k}(p_{\theta}(x)) \nabla p_{\theta}(x)]$ as shown in \eqref{eq:pass@k-grad}, the chain rule gives the Hessian of $J_k$ for every $\theta \in \mathbb{R}^d$,  
\begin{equation}
\label{eq:hess-chain}
\nabla^2 J_k(\theta)
=
\mathbb{E}_{x\sim\mathcal D}\Big[
\omega_k'(p_{\theta}(x))\,\nabla p_{\theta}(x)\,\nabla p_{\theta}(x)^\top
+
\omega_k(p_{\theta}(x))\,\nabla^2 p_{\theta}(x)
\Big].
\end{equation}
Taking the operator norm and using the triangular inequality yields: 
\begin{equation}
\label{eq:hess-norm-split}
\big\|\nabla^2 J_k(\theta)\big\|
\le
\mathbb{E}_{x\sim\mathcal D}\Big[
\big|\omega_k'(p_{\theta}(x))\big| \cdot\,\|\nabla p_{\theta}(x)\|^2
+
\big|\omega_k(p_{\theta}(x))\big|\cdot\,\|\nabla^2 p_{\theta}(x)\|
\Big].
\end{equation}
In the rest of the proof, we bound each term separately. 

\paragraph{Step 1: Bounding $\|\nabla p_{\theta}(x)\|^2$.} 
Recall from \eqref{eq:pass@1pg} that the pass@1 policy gradient is given by:
\begin{equation}
\nabla p_{\theta}(x)
=
\mathbb{E}_{y\sim\pi_\theta(\cdot |x)}\!\Big[r(x,y)\,\nabla \log \pi_\theta(y|x)\Big].
\end{equation}
By Jensen's inequality and since $r(x,y)\in\{0,1\}$, we have
\begin{align}
\|\nabla p_{\theta}(x)\|^2
&\le
\mathbb{E}_{y\sim\pi_\theta(\cdot |x)}\Big[r(x,y)^2\,\|\nabla \log \pi_\theta(y|x)\|^2\Big]
\le
\mathbb{E}_{y\sim\pi_\theta(\cdot |x)}\Big[\|\nabla \log \pi_\theta(y|x)\|^2\Big]
\le G^2\,,
\label{eq:grad-e-bound}
\end{align}
where the last bound follows from Assumption~\ref{as:E-LS}. 

\paragraph{Step 2: Bounding $\|\nabla^2 p_{\theta}(x)\|$.}
Differentiating the pass@1 policy gradient, we obtain: 
\[
\nabla^2 p_{\theta}(x)
=
\mathbb{E}_{y\sim\pi_\theta(\cdot |x)}\!\Big[
r(x,y)\,\nabla \log \pi_\theta(y|x)\,\nabla \log \pi_\theta(y|x)^\top
+
r(x,y)\,\nabla^2 \log \pi_\theta(y|x)
\Big].
\]
Taking the operator norm and using the triangle inequality, we obtain: 
\begin{align}
\|\nabla^2 p_{\theta}(x)\|
&\le
\mathbb{E}_{y\sim\pi_\theta(\cdot|x)}\Big[
|r(x,y)| \cdot \,\|\nabla \log \pi_\theta(y|x)\|^2
\Big]
+
\mathbb{E}_{y\sim\pi_\theta(\cdot |x)}\Big[
|r(x,y)|\cdot \,\|\nabla^2 \log \pi_\theta(y|x)\|
\Big]
\nonumber\\
&\le
\mathbb{E}_{y\sim\pi_\theta(\cdot |x)}\Big[\|\nabla \log \pi_\theta(y|x)\|^2\Big]
+
\mathbb{E}_{y\sim\pi_\theta(\cdot |x)}\Big[\|\nabla^2 \log \pi_\theta(y|x)\|\Big]
\le G^2 + F\,,
\label{eq:hess-e-bound}
\end{align}
where the last inequality again follows from using Assumption~\ref{as:E-LS}. 

\paragraph{Step 3: Bounding weights $\omega_k'$ and $\omega_k$.} By definition of the weights $w_k(p) = f_k'(p) = k(1-p)^{k-1}$ for any $p\in[0,1]$. Therefore it follows that $\omega_k'(p)=-k(k-1)(1-p)^{k-2}$ for $k \geq 2$. Hence, for any $p \in [0,1]$ and any $k \geq 2$,
\begin{equation}
\label{eq:weights-bounds}
|\omega_k(p)| \le k\,, \quad |\omega_k'(p)| \le k(k-1)\,. 
\end{equation}
For $k=1$, $w_1(p)=1$ and $w_1'(p)=0$, so the same conclusion follows directly.

\paragraph{Conclusion.}
Plugging \eqref{eq:grad-e-bound}--\eqref{eq:hess-e-bound} and the bounds on $\omega_k,\omega_k'$ \eqref{eq:weights-bounds} into \eqref{eq:hess-norm-split} yields
\[
\|\nabla^2 J_k(\theta)\|
\le
k(k-1)\,G^2 + k\,(G^2+F)
\le k^2 G^2 + kF.
\]
Finally, boundedness of the Hessian norm implies that $J_k$ is
$L_k$-smooth with $L_k=k^2G^2+kF$, i.e., for any $\theta, \theta' \in \mathbb{R}^d\,,$
\[
\Big|J_k(\theta')-J_k(\theta)-\langle \nabla J_k(\theta),\theta'-\theta\rangle\Big|
\le \frac{L_k}{2}\|\theta'-\theta\|^2.
\]
\end{proof}

\subsection{Proof of Proposition~\ref{prop:pass@1-degradation}: Pass@$1$ degradation under pass@$k$ updates}

\begin{proof}
We prove the decrease of pass@1 and increase of pass@$k$ under pass@$k$ policy gradient ascent separately. Recall the policy update rule with step size $\eta > 0$: 
\begin{equation}
\label{eq:pass@k-gd}
\theta^{+} = \theta + \eta\,\nabla J_k(\theta).
\end{equation}

\noindent\textbf{(i) Pass@1 descent.}
By $L_1$-smoothness of $J_1$ (Lemma~\ref{lem:smoothness-pass@k}) and using the policy update rule~\eqref{eq:pass@k-gd}, we have: 
\begin{align}
J_1(\theta^{+})
&\le J_1(\theta) + \langle \nabla J_1(\theta),\theta^{+}-\theta\rangle
    + \frac{L_1}{2}\|\theta^{+}-\theta\|^2 \nonumber\\
&= J_1(\theta) + \eta\,\langle \nabla J_1(\theta),\nabla J_k(\theta)\rangle
    + \frac{L_1\eta^2}{2}\|\nabla J_k(\theta)\|^2 .
\label{eq:smooth-J1}
\end{align}

We now bound each one of the last two terms in the above inequality.

\noindent\textbf{Bounding $\|\nabla J_k(\theta)\|^2$.}
Recall the pass@$k$ policy gradient expression $\nabla J_k(\theta)
= \mathbb{E}_{x}\!\left[\omega_\theta(x)\,\nabla J_1(x;\theta)\right]$, we have by Jensen and Cauchy-Schwarz inequalities: 
\begin{equation}
\|\nabla J_k(\theta)\|^2
= \left\|\mathbb{E}_{x}\!\left[\omega_\theta(x)\,\nabla J_1(x;\theta)\right]\right\|^2
\;\le\;
\mathbb{E}_{x}\!\left[|\omega_\theta(x)|^2\,\|\nabla J_1(x;\theta)\|^2\right] \leq k^2 G^2\,,
\label{eq:grad-bound-1}
\end{equation}
where the last inequality follows from using the pass@1 policy gradient bound in \eqref{eq:pass@1-grad-bound} together with immediate weight bound $|\omega_\theta(x)|\le k$ and the fact that the reward is binary (this can be generalized to a continuous uniformly bounded reward function). 

\noindent\textbf{Bounding $\langle \nabla J_1(\theta),\nabla J_k(\theta)\rangle$.} 
By Corollary~\ref{cor:dominating-neg-interfer-prompts}, we have 
\begin{equation}
\label{eq:innerprod-bound}
\langle \nabla J_1(\theta),\nabla J_k(\theta)\rangle \leq - \delta(\theta)\,.
\end{equation}
Combining the bounds~\eqref{eq:grad-bound-1} and \eqref{eq:innerprod-bound}, we immediately obtain from \eqref{eq:smooth-J1}: 
\begin{equation}
J_1(\theta^{+})
\leq J_1(\theta) -  \eta \delta(\theta) + C_2 \eta^2\,,
\end{equation}
where $C_2 := \frac{L_1 k^2 G^2}{2}$ and $L_1 = G^2 + F$ as shown in Lemma~\ref{lem:smoothness-pass@k}, $G, F$ being smoothness constants defined in Assumption~\ref{as:E-LS}. 

Hence, if $\delta(\theta) > 0$ and $0 < \eta < \eta_+ := \frac{\delta(\theta)}{C_2},$ then $-\eta \delta(\theta) + C_2\eta^2<0$ and it follows that 
\begin{equation}
J_1(\theta^{+})
\leq J_1(\theta) -  \eta \delta(\theta) + C_2 \eta^2 < J_1(\theta)\,,
\end{equation}
which is the first desired inequality. 

\noindent\textbf{(ii) Pass@$k$ ascent.} As for pass@$k$ under pass@$k$ policy gradient ascent, the proof is a standard application of smoothness in optimization. We provide a brief proof in our setting for completeness. 

Using $L_k$-smoothness of the pass@$k$ objective shown in Lemma~\ref{lem:smoothness-pass@k} together with the policy update rule~\eqref{eq:pass@k-gd}, we have: 
\begin{align}
J_k(\theta^+)
&\ge J_k(\theta) + \langle \nabla J_k(\theta),\theta^+-\theta\rangle
     -\frac{L_k}{2}\|\theta^+-\theta\|^2 \nonumber\\
&= J_k(\theta) + \eta\|\nabla J_k(\theta)\|^2
     -\frac{L_k\eta^2}{2}\|\nabla J_k(\theta)\|^2 \nonumber\\
&= J_k(\theta) + \eta\Bigl(1-\frac{L_k\eta}{2}\Bigr)\|\nabla J_k(\theta)\|^2.
\label{eq:smooth-Jk-lower}
\end{align}
We pick then $\eta \leq \frac{1}{L_k}$ to obtain the second desired inequality: 
\begin{equation}
J_k(\theta^+) \geq  J_k(\theta) + \frac{\eta}{2} \|\nabla J_k(\theta)\|^2 \geq  J_k(\theta)\,.
\end{equation}
Overall, a sufficient step-size condition to guarantee simultaneous increase of pass@$k$ and decrease of pass@1 (i.e., in particular $J_1(\theta^+)<J_1(\theta)$ and $J_k(\theta^+)>J_k(\theta)$ when $\nabla J_k(\theta) \neq 0$) is:  
\begin{equation}
0 < \eta < \min\left(\frac{\delta(\theta)}{C_2}, \frac{1}{L_k}\right)\,,
\end{equation}
and this concludes the proof. 
\end{proof}

\newpage 
\section{Additional Experiments}
\label{appx:additional-exp}

\subsection{More details regarding experimental setting}

\noindent\textbf{Sampling configuration.} For each problem, we generate $k=32$ independent responses using temperature sampling with temperature $T=0.7$ and nucleus sampling with $p=0.95$. Responses are evaluated using exact match against ground truth answers and binary rewards indicate correctness. 

\noindent\textbf{Experiments compute resources.} Experiments on the toy example in Appendix~\ref{sec:toy-example} were performed on a MacBook Pro M4 and only take a few seconds to minutes. Experiments on LLMs were performed on GPUs and take several hours to run. 

\subsection{Threshold robustness analysis}
\label{app:threshold_robustness}

Figure~\ref{fig:all_threshold_configs_llama} shows gradient misalignment across all additional threshold configurations for DeepSeek-R1-Distill-Llama-8B (beyond the main text example of $\delta_1=0.85$, $\delta_2=0.10$). All combinations exhibit negative inner products (range: $-0.49$ to $-0.65$), demonstrating that the misalignment phenomenon is robust to threshold choices.

\begin{figure}[p]
  \centering
  \begin{subfigure}{0.9\textwidth}
    \centering
    \includegraphics[width=\textwidth]{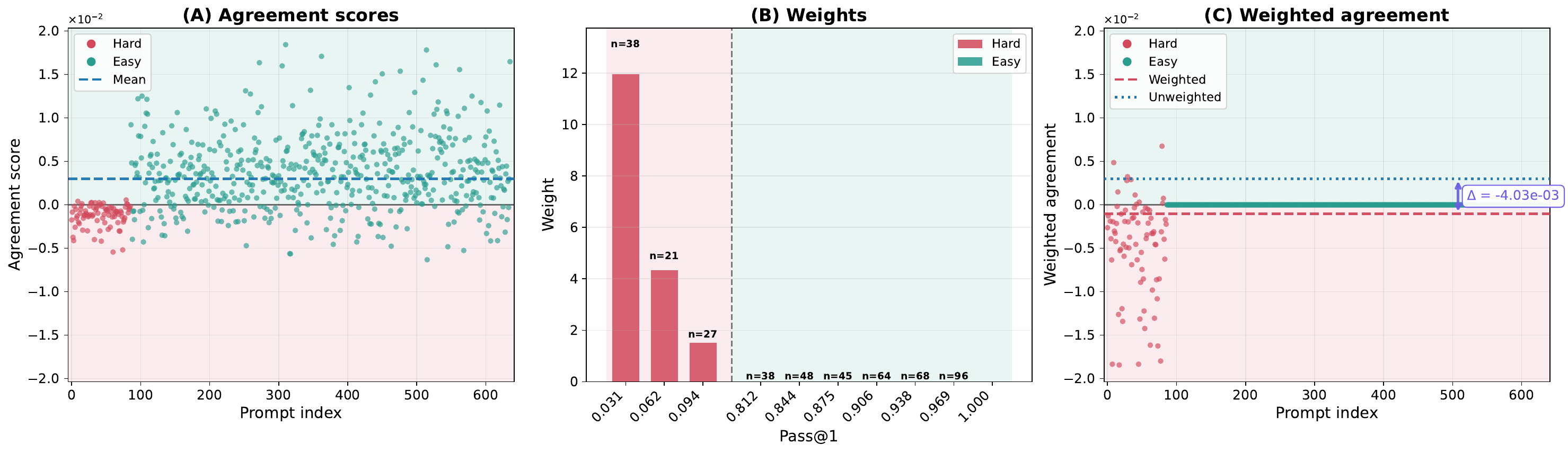}
    \caption{$\delta_1=0.80$, $\delta_2=0.10$ (636 prompts: 86 hard, 550 easy, IP=$-0.652$)}
  \end{subfigure}
  
  \vspace{0.5em}
  
  \begin{subfigure}{0.9\textwidth}
    \centering
    \includegraphics[width=\textwidth]{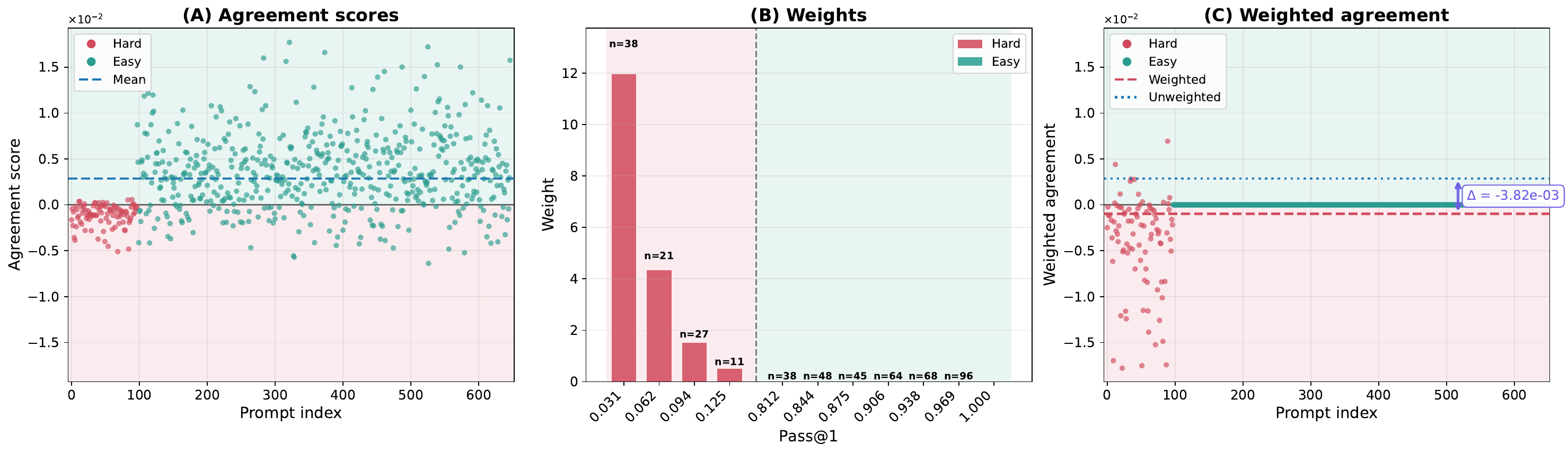}
    \caption{$\delta_1=0.80$, $\delta_2=0.15$ (647 prompts: 97 hard, 550 easy, IP=$-0.621$)}
  \end{subfigure}
  
  \vspace{0.5em}
  
  \begin{subfigure}{0.85\textwidth}
    \centering
    \includegraphics[width=\textwidth]{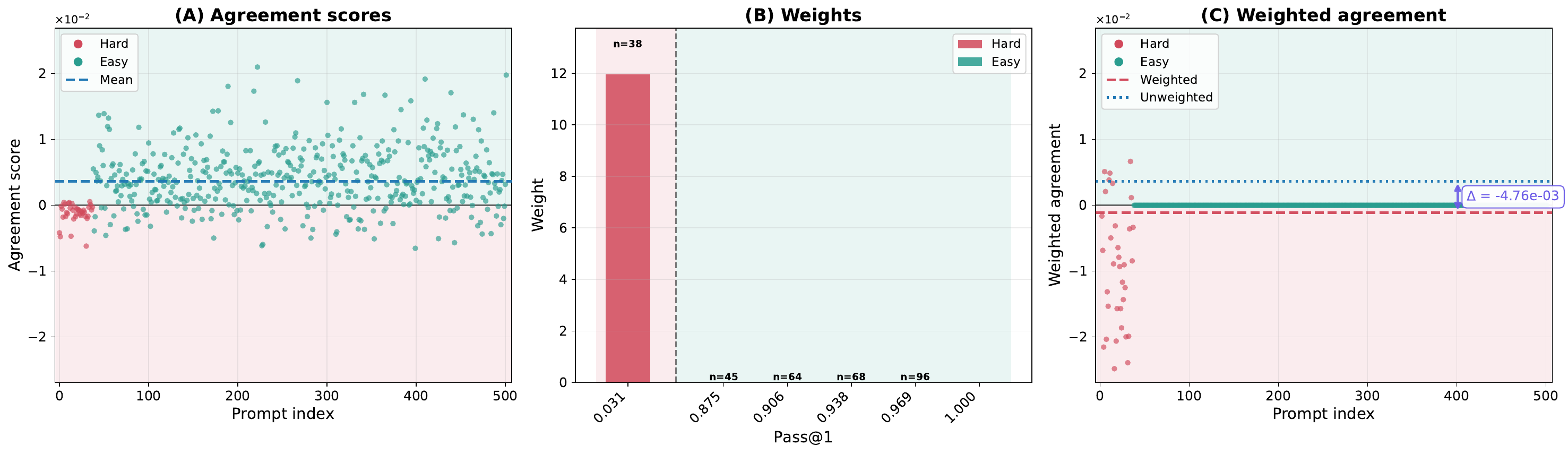}
    \caption{$\delta_1=0.85$, $\delta_2=0.05$ (502 prompts: 38 hard, 464 easy, IP=$-0.551$)}
  \end{subfigure}
  
  \vspace{0.5em}
  
  \begin{subfigure}{0.85\textwidth}
    \centering
    \includegraphics[width=\textwidth]{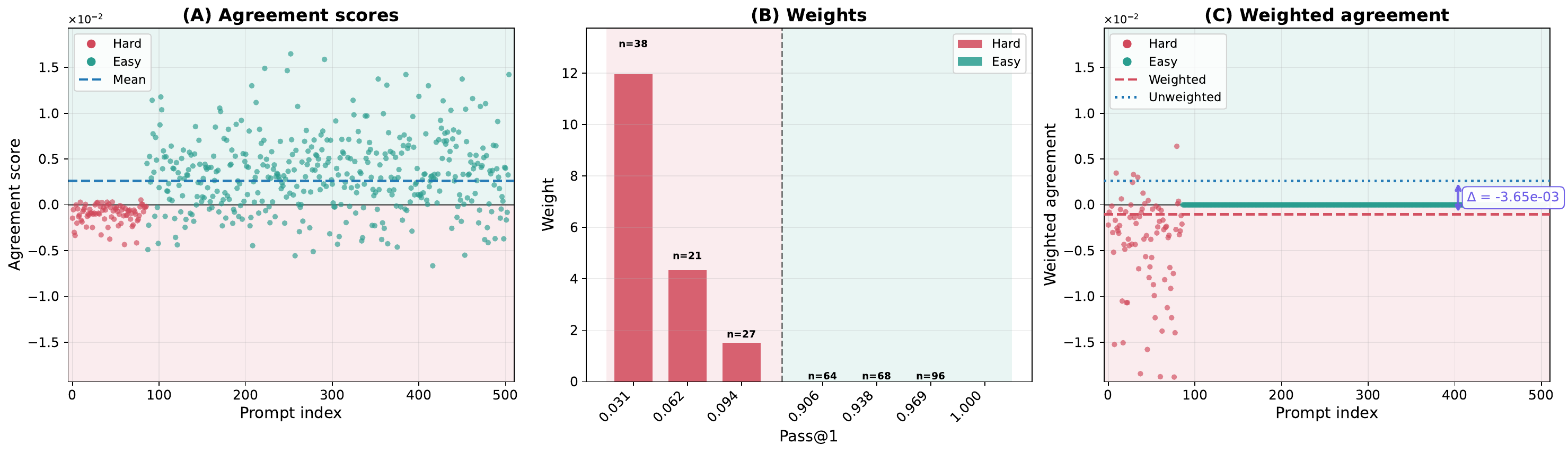}
    \caption{$\delta_1=0.90$, $\delta_2=0.10$ (505 prompts: 86 hard, 419 easy, IP=$-0.527$)}
  \end{subfigure}
  
 \caption{
  \textbf{Gradient conflict across all threshold configurations (DeepSeek-R1-Distill-Llama-8B).}
  Each panel shows agreement scores, weights, and weighted contributions for a different $(\delta_1, \delta_2)$ configuration. All exhibit negative inner products, confirming robust misalignment across thresholds. The consistent downward shift from unweighted to weighted means shows pass@$k$ optimization systematically opposes pass@1 improvement.
}
  \label{fig:all_threshold_configs_llama}
\end{figure}

\subsection{Additional ablations: Which prompts are amplified by pass@$k$?}

To have a complementary view of the amplification mechanism, we directly plot weights against agreement scores, with color encoding pass@1 probability. See Figure~\ref{fig:weight_vs_alignment_scatter} which shows results for both models across representative threshold configurations. The critical observation is where these weights concentrate in agreement-score space (x-axis). High-weight prompts cluster in the negative-agreement region ($a_\theta(x) < 0$), while low-weight prompts cluster in the positive-agreement region. This means pass@$k$'s design, which correctly prioritizes hard prompts for its own objective, systematically amplifies negatively interfering prompts that hurt pass@1.

\begin{figure*}[!h]
  \centering
  \begin{subfigure}{0.32\textwidth}
    \centering
    \includegraphics[width=\textwidth,height=0.25\textheight,keepaspectratio]{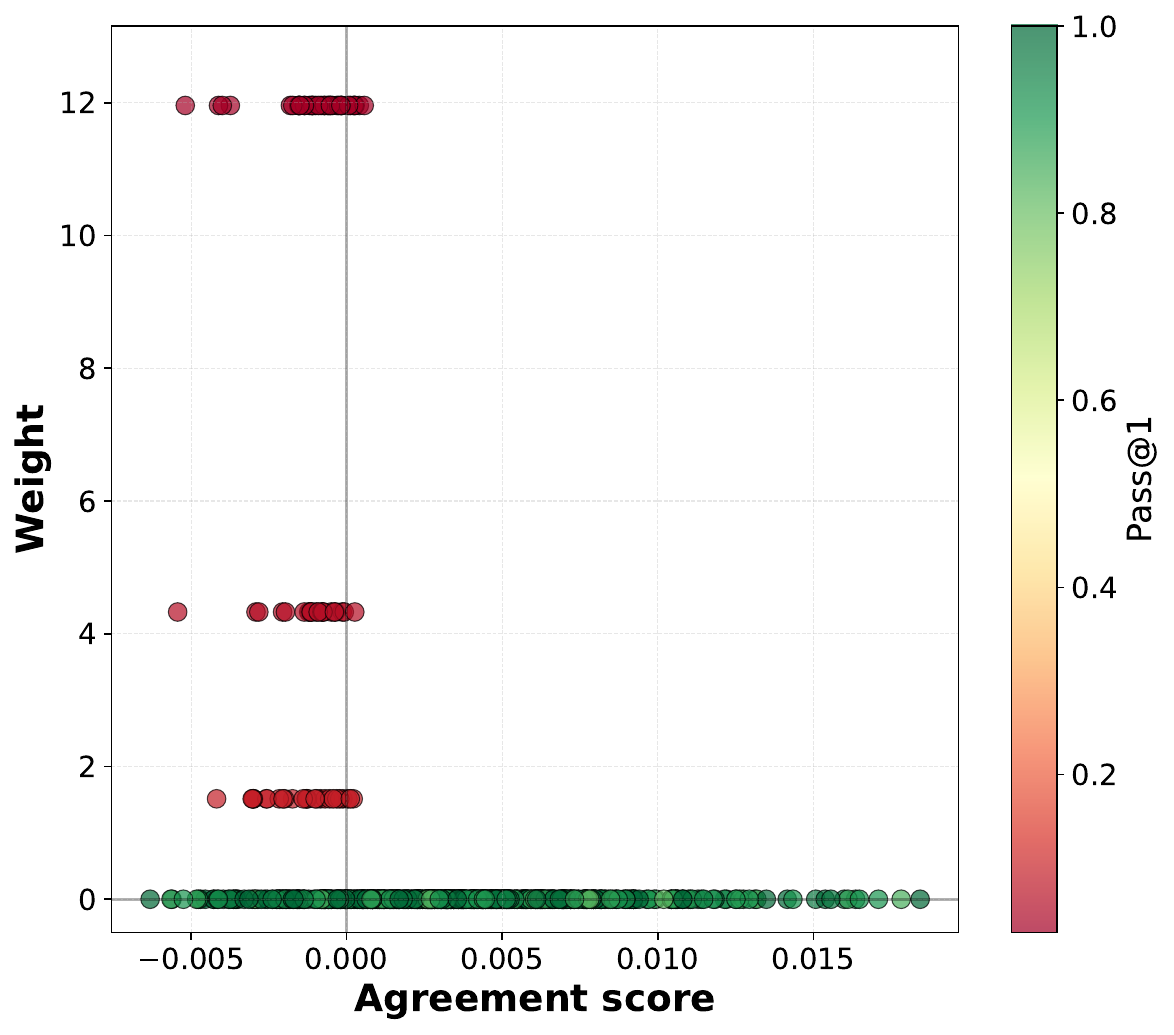}
    \caption{Llama: $\delta_1=0.80$, $\delta_2=0.10$}
    \label{fig:llama_080_010_weight}
  \end{subfigure}
  \hfill
  \begin{subfigure}{0.32\textwidth}
    \centering
    \includegraphics[width=\textwidth,height=0.25\textheight,keepaspectratio]{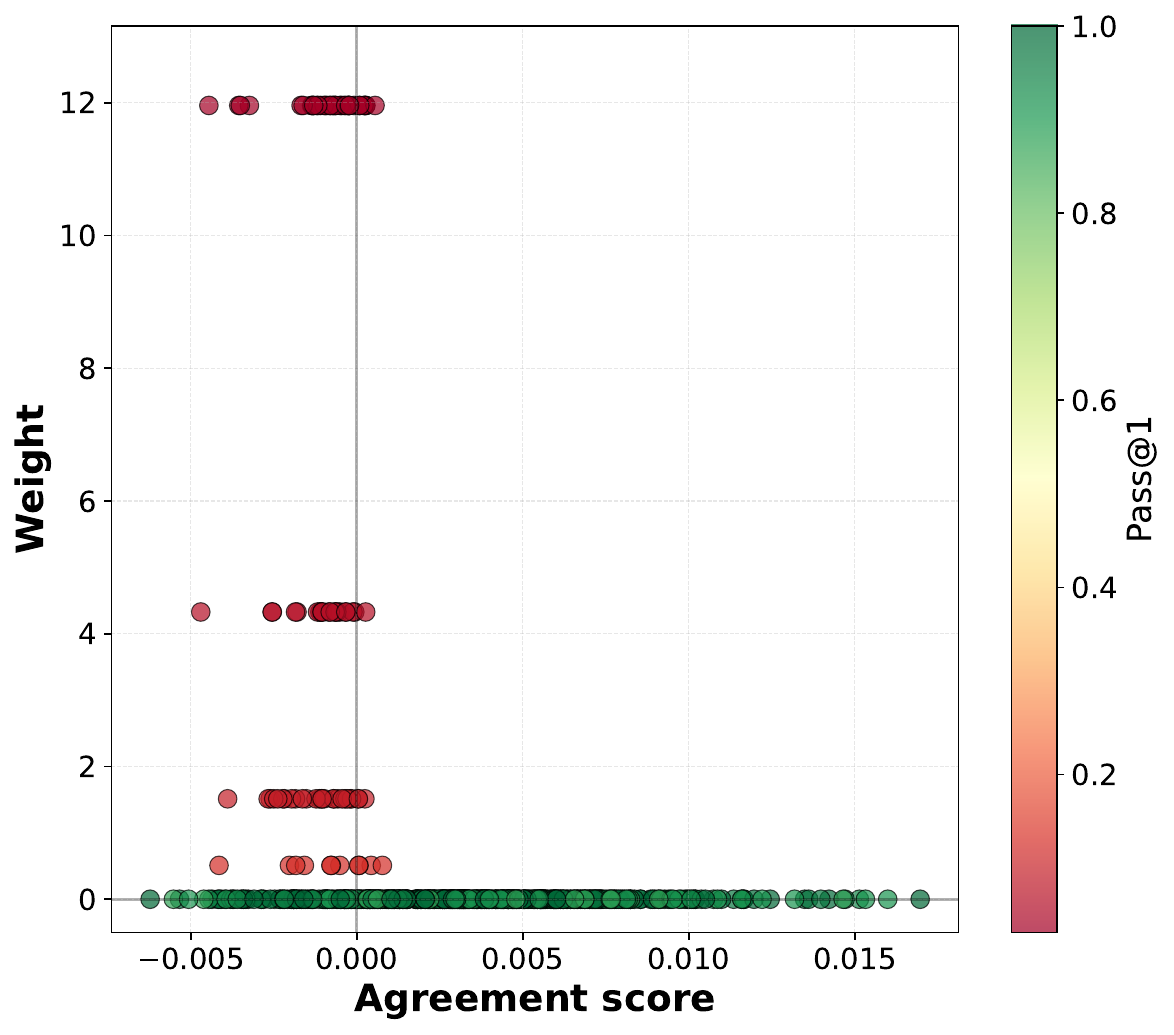}
    \caption{Llama: $\delta_1=0.85$, $\delta_2=0.15$}
    \label{fig:llama_085_015_weight}
  \end{subfigure}
  \hfill
  \begin{subfigure}{0.32\textwidth}
    \centering
    \includegraphics[width=\textwidth,height=0.25\textheight,keepaspectratio]{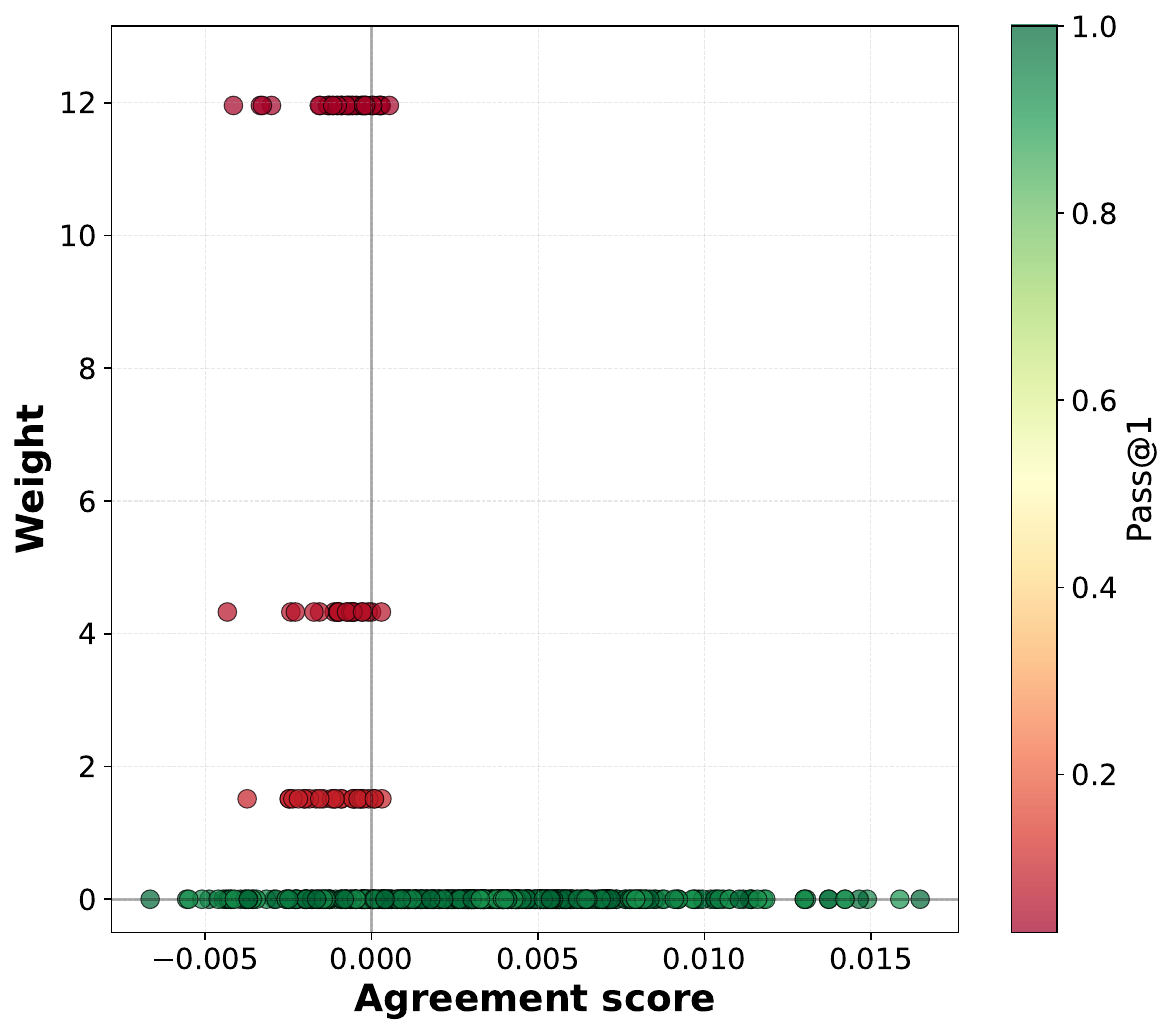}
    \caption{Llama: $\delta_1=0.90$, $\delta_2=0.10$}
    \label{fig:llama_085_015_weight}
  \end{subfigure}
  
  \vspace{1.5em}
  
  \begin{subfigure}{0.32\textwidth}
    \centering
    \includegraphics[width=\textwidth,height=0.25\textheight,keepaspectratio]{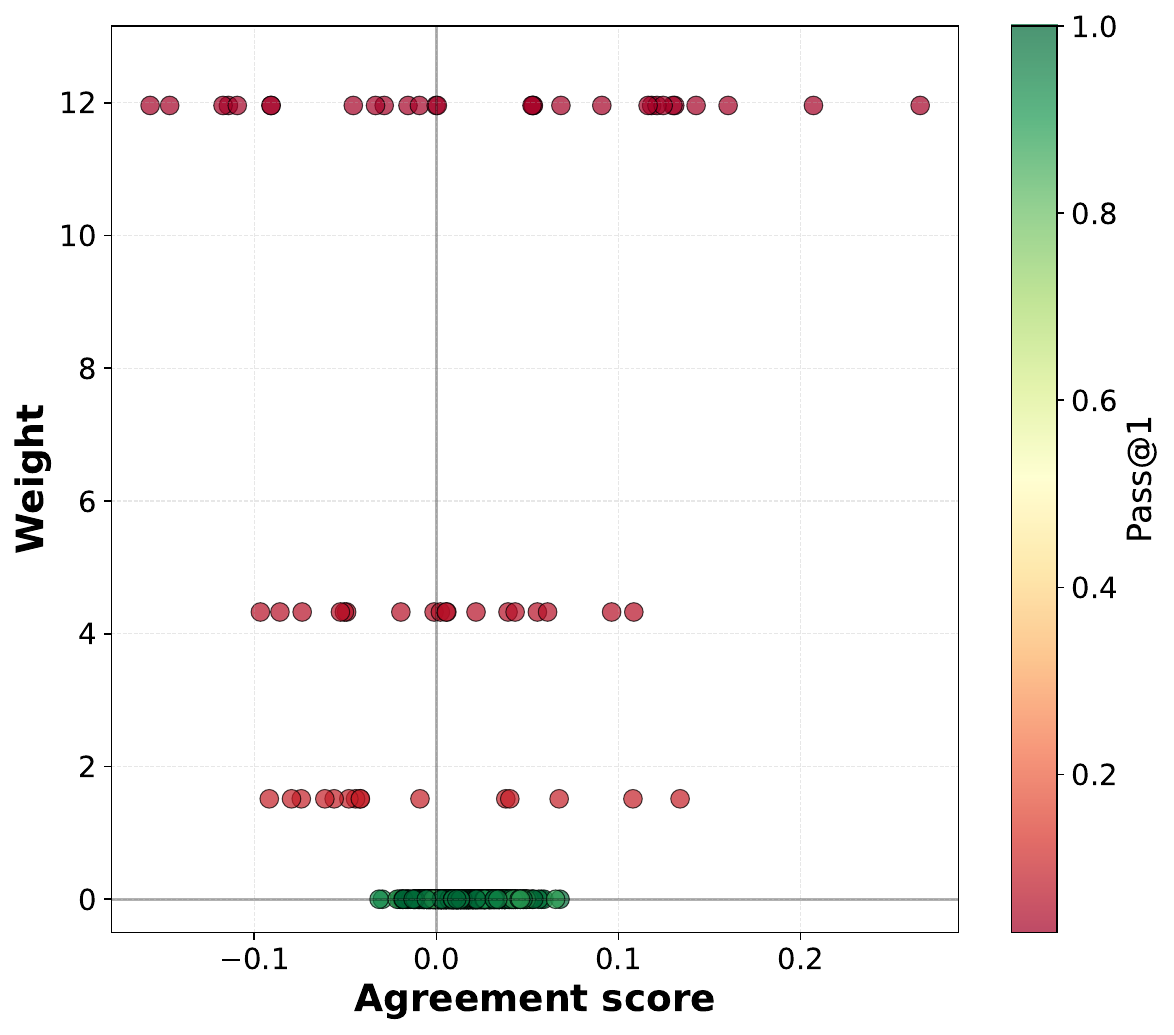}
    \caption{Qwen: $\delta_1=0.80$, $\delta_2=0.10$}
    \label{fig:qwen_080_010_weight}
  \end{subfigure}
  \hfill
  \begin{subfigure}{0.32\textwidth}
    \centering
    \includegraphics[width=\textwidth,height=0.25\textheight,keepaspectratio]{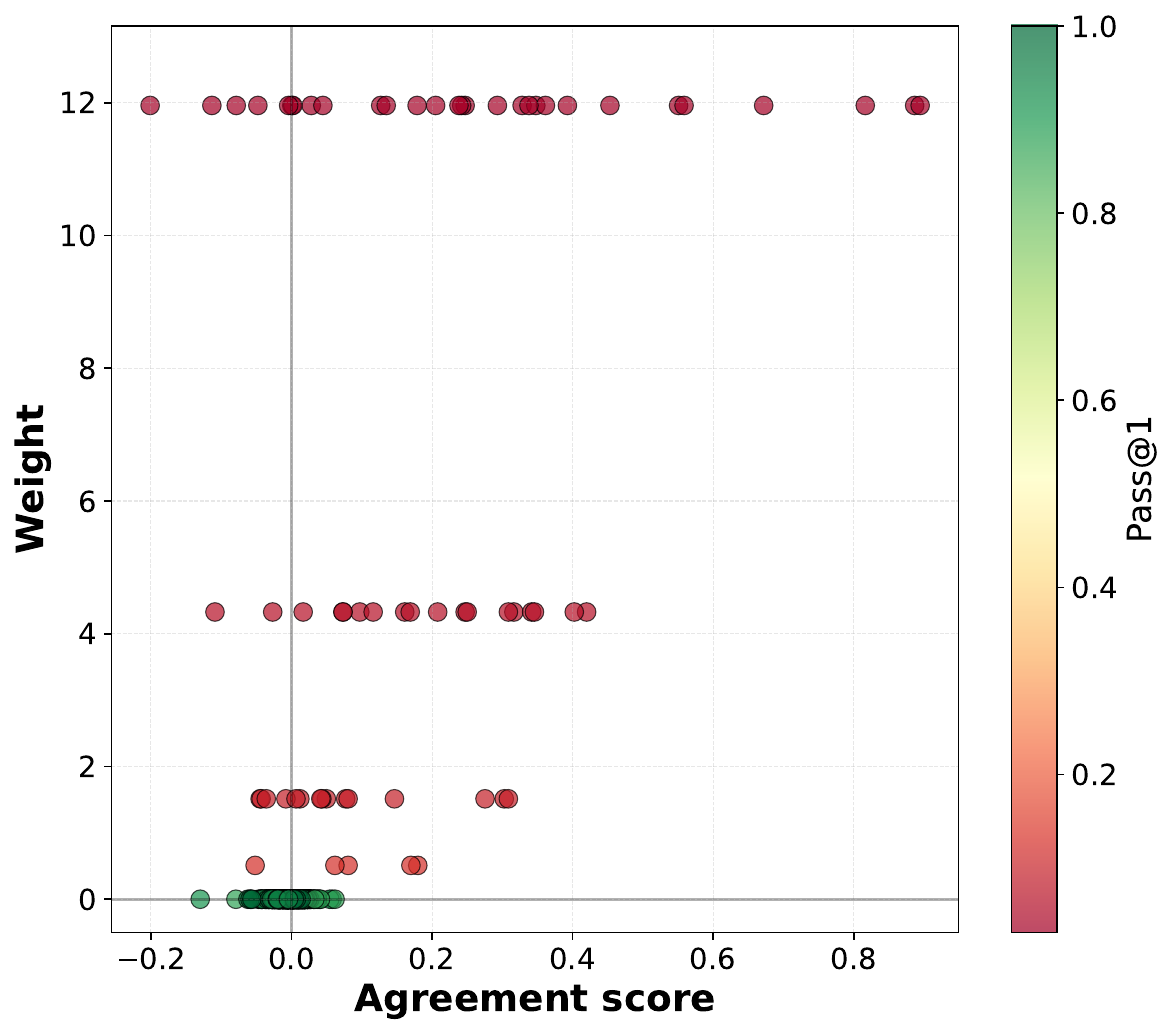}
    \caption{Qwen: $\delta_1=0.85$, $\delta_2=0.15$}
    \label{fig:qwen_085_015_weight}
  \end{subfigure}
  \hfill
  \begin{subfigure}{0.32\textwidth}
    \centering
    \includegraphics[width=\textwidth,height=0.25\textheight,keepaspectratio]{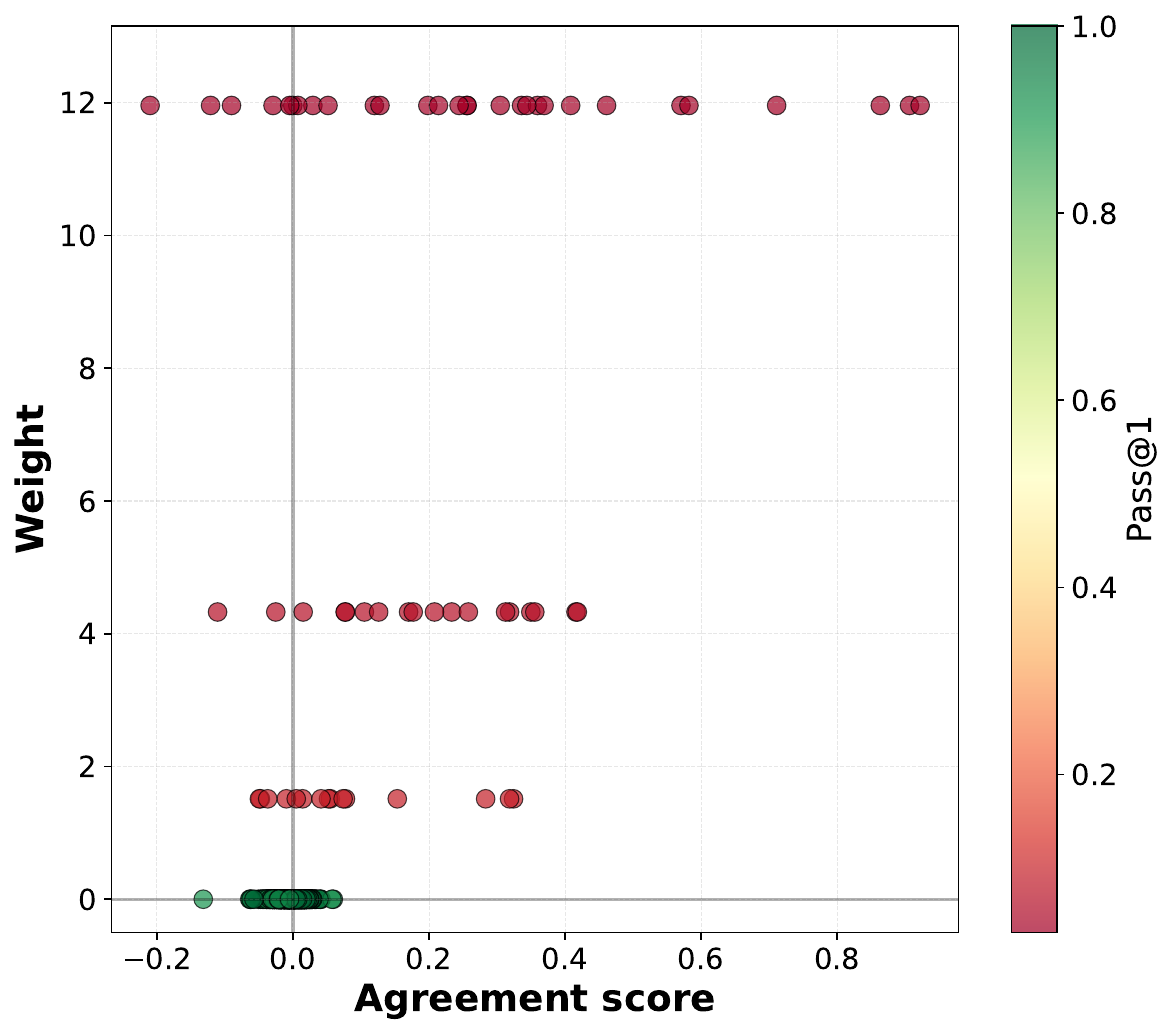}
    \caption{Qwen: $\delta_1=0.90$, $\delta_2=0.10$}
    \label{fig:qwen_090_010_weight}
  \end{subfigure}
  
  \caption{
    \textbf{Pass@$k$ amplification pattern across models and thresholds.}
    Each point represents one prompt, with the corresponding agreement score (x-axis), pass@$k$ weight (y-axis), and pass@1 probability shown by color (red = low, green = high).
    \textbf{Consistent pattern:} Prompts with negative alignment (left), low pass@1 (red), and high weights (4-12) cluster upper-left, while prompts with positive alignment (right), high pass@1 (green), and zero weights ($\sim$10$^{-28}$) appear lower-right.
    \textbf{Top row (a-c):} DeepSeek-R1-Distill-Llama-8B across three $\delta_2$ values (0.05, 0.10, 0.15).
    \textbf{Bottom row (d-f):} DeepSeek-R1-Distill-Qwen-7B with same thresholds.
  }
  \label{fig:weight_vs_alignment_scatter}
\end{figure*}

\newpage
\subsection{Additional experiments on code verification}
\label{app:code-generation}

\begin{figure}[h]
    \centering

    \begin{subfigure}[b]{0.48\textwidth}
        \centering
        \includegraphics[width=\linewidth]{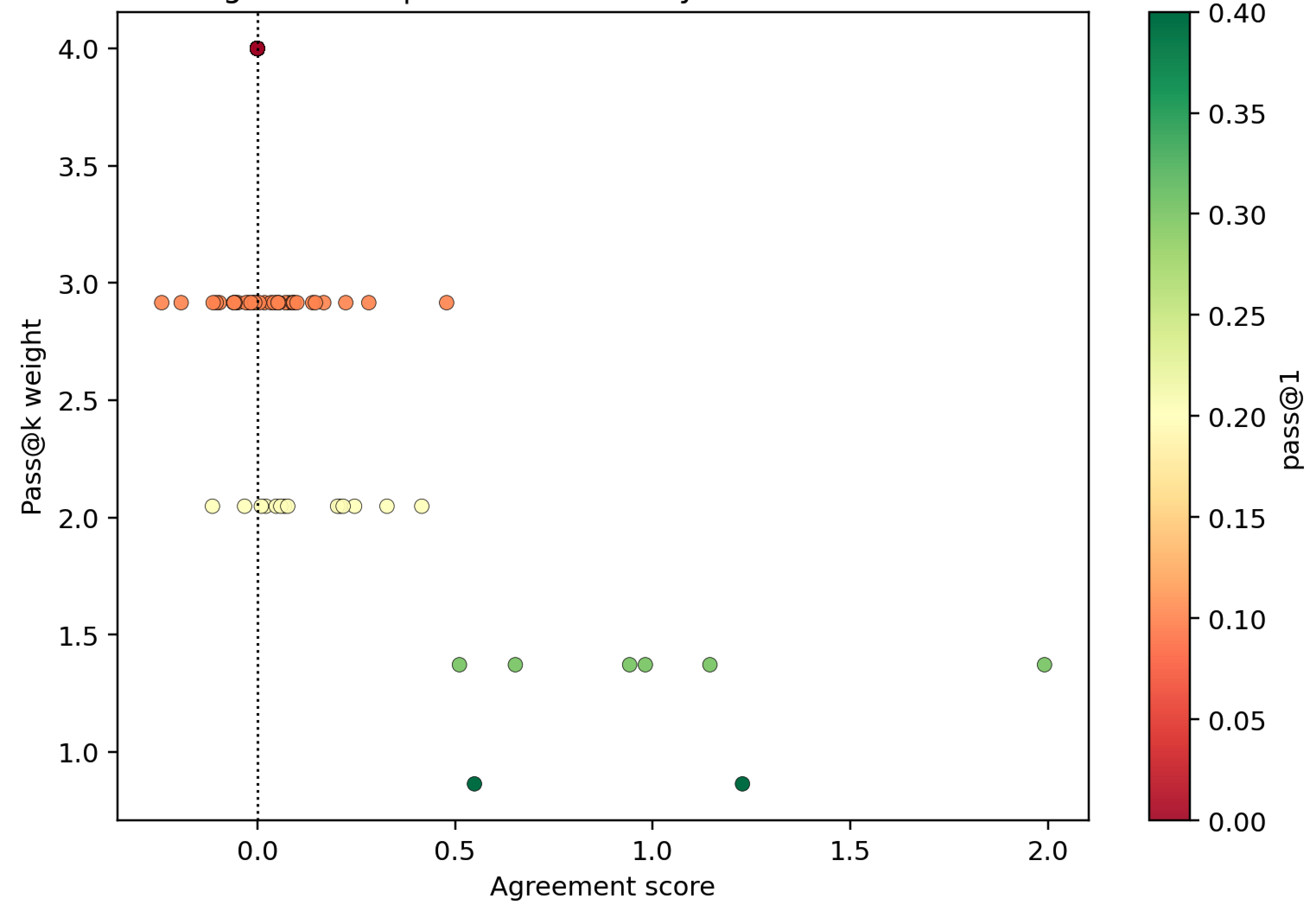}
        \caption{$k=4$}
        \label{fig:humaneval-k4}
    \end{subfigure}
    \hfill
    \begin{subfigure}[b]{0.48\textwidth}
        \centering
        \includegraphics[width=\linewidth]{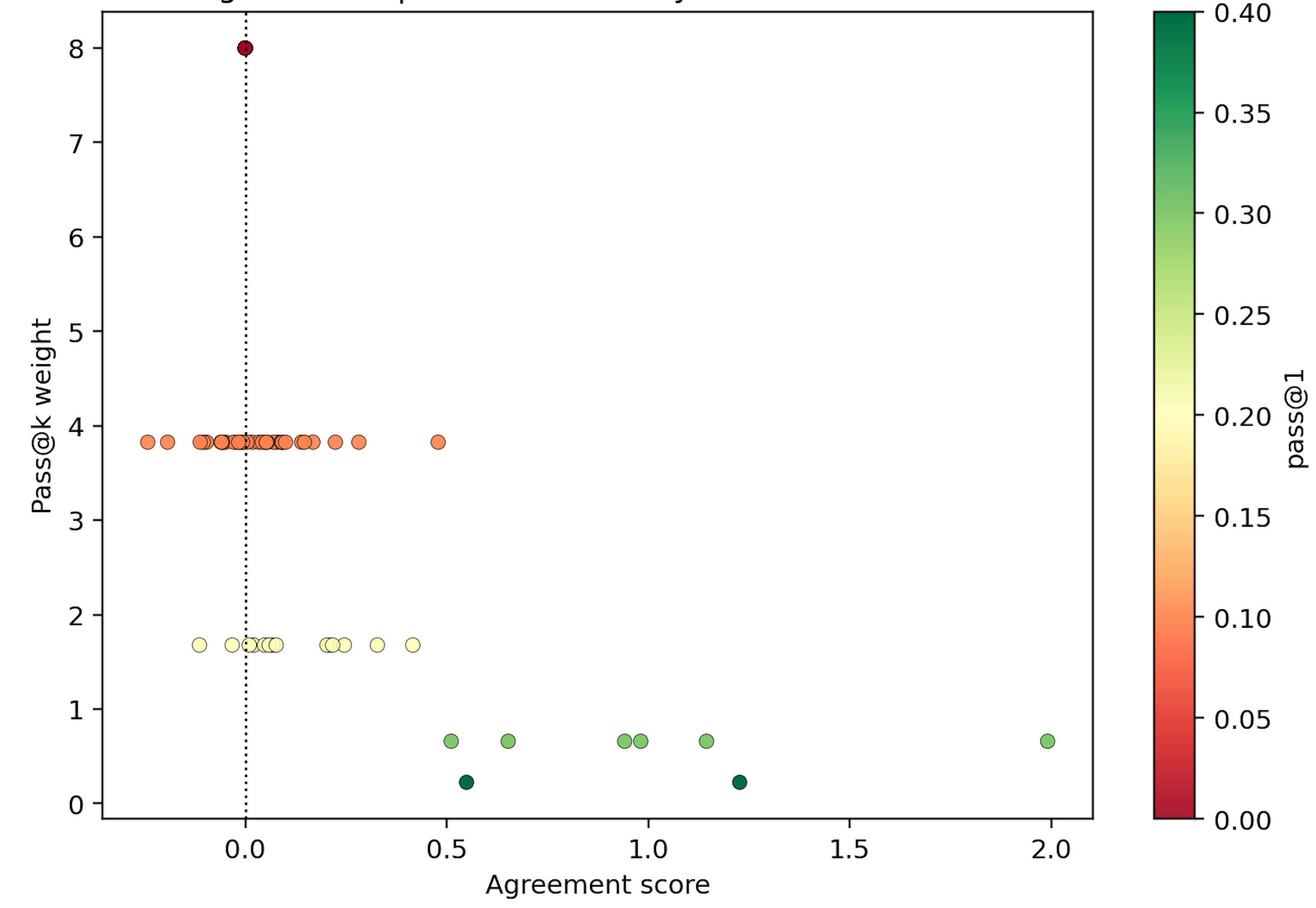}
        \caption{$k=8$}
        \label{fig:humaneval-k8}
    \end{subfigure}

    \vspace{0.5em}

    \begin{subfigure}[b]{0.48\textwidth}
        \centering
        \includegraphics[width=\linewidth]{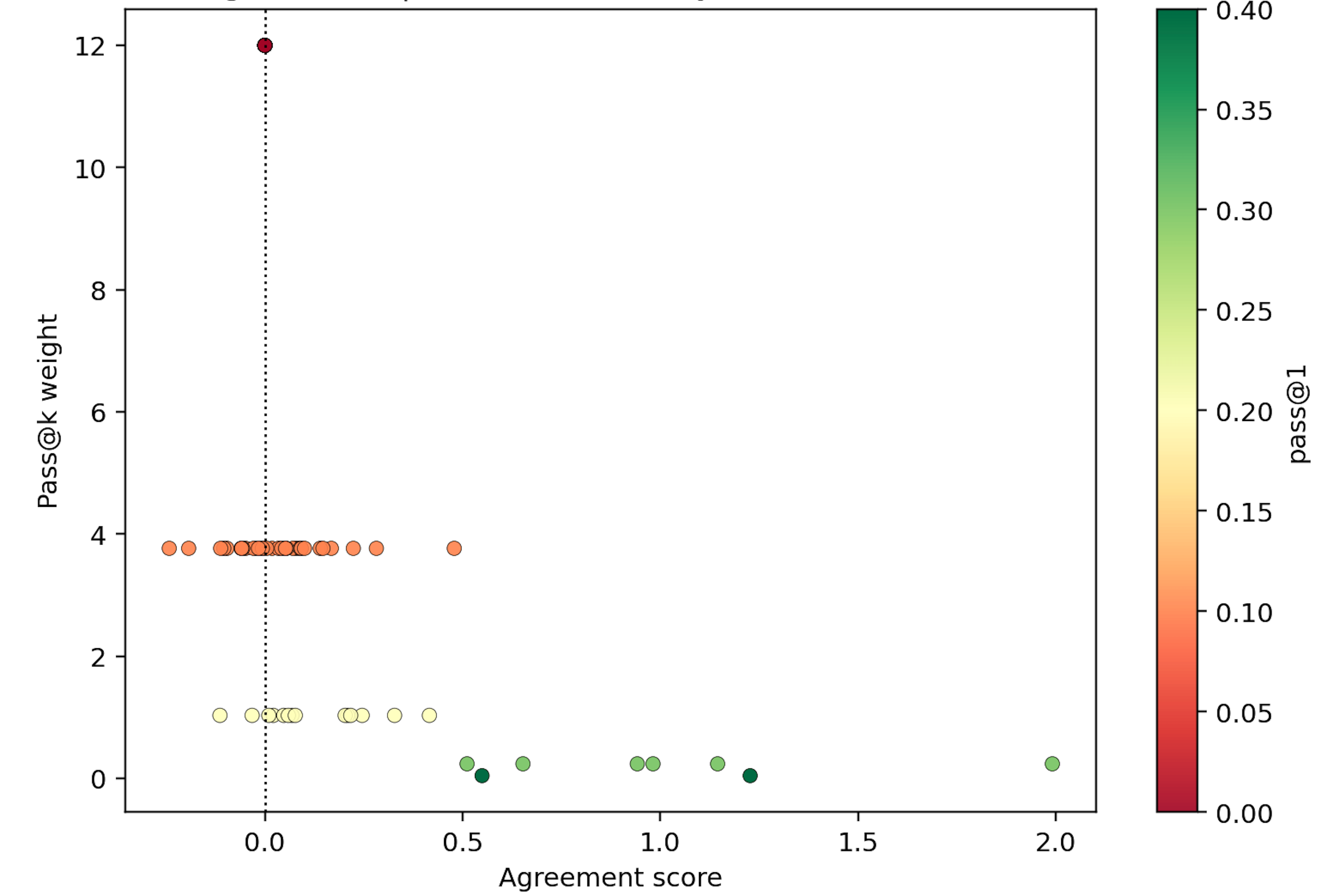}
        \caption{$k=12$}
        \label{fig:humaneval-k12}
    \end{subfigure}
    \hfill
    \begin{subfigure}[b]{0.48\textwidth}
        \centering
        \includegraphics[width=\linewidth]{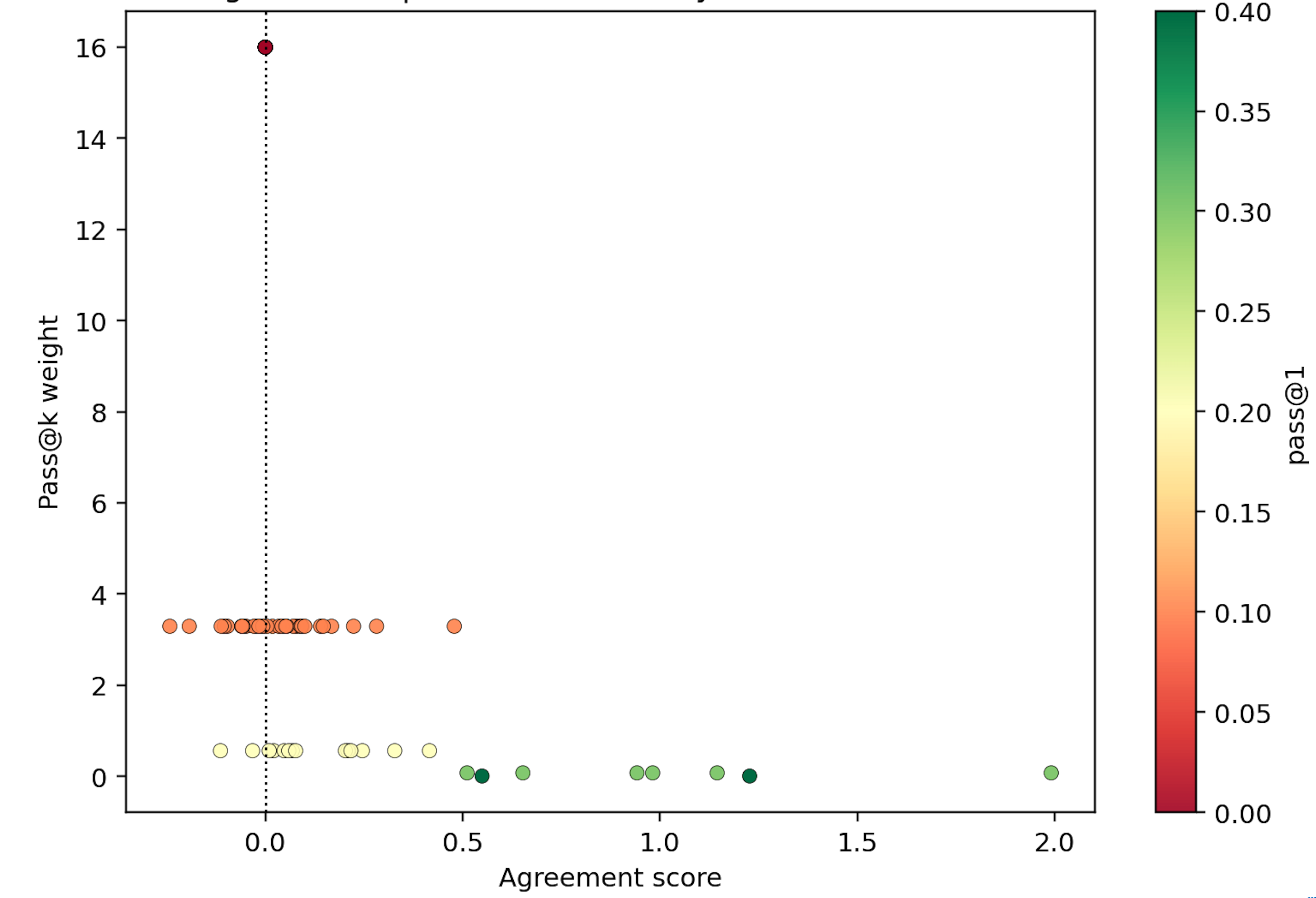}
        \caption{$k=16$}
        \label{fig:humaneval-k16}
    \end{subfigure}

    \caption{\textbf{Pass@$k$ amplification and negative prompt interference in HumanEval dataset using a Qwen 1.5B model for different values of $k$.}
    Each point represents one prompt, with the corresponding agreement score (x-axis), pass@$k$ weight (y-axis), and pass@1 probability shown by color (red = low, green = high).
    \textbf{Consistent pattern:} Prompts with negative alignment (left), low pass@1 (red), and high weights (4-12) cluster upper-left, while prompts with positive alignment (right), high pass@1 (green), and small weights appear lower-right.}
    \label{fig:humaneval-all}
\end{figure}

\begin{table}[htbp]
\centering
\caption{Negatively interfering prompt pairs on code generation benchmarks (Definition 3.1).}
\label{tab:r1}
\begin{tabular}{llrrrr}
\toprule
Model & Dataset & \# Prompts & Neg.\ Pairs & Fraction \\
\midrule
Qwen2.5-Coder-1.5B-Instruct & HumanEval & 164  & \textbf{744} & \textbf{5.6\%} \\
Meta-Llama-3.1-8B & MBPP & 60  & \textbf{412} & \textbf{23.3\%} \\
\bottomrule
\end{tabular}
\end{table}

\begin{table}[htbp]
\centering
\caption{Top-5 negatively interfering pairs --- Qwen2.5-Coder-1.5B-Instruct on HumanEval.}
\label{tab:r2}
\begin{tabular}{llrrrr}
\toprule
Prompt $i$ & Prompt $j$ & $\langle g_i,g_j\rangle$ & $\cos(g_i,g_j)$ & $\hat{p}_1(i)$ & $\hat{p}_1(j)$ \\
\midrule
HumanEval/92  & HumanEval/44  & $-1594.8$ & $-0.198$ & 1.000 & 1.000 \\
HumanEval/72  & HumanEval/104 & $-1450.5$ & \textbf{$-0.255$} & 0.938 & 0.688 \\
HumanEval/72  & HumanEval/133 & $-1290.0$ & $-0.167$ & 0.938 & 0.875 \\
HumanEval/92  & HumanEval/73  & $-1267.6$ & $-0.136$ & 1.000 & 1.000 \\
HumanEval/72  & HumanEval/143 & $-1246.8$ & $-0.112$ & 0.938 & 0.938 \\
\bottomrule
\end{tabular}
\end{table}

\begin{table}[htbp]
\centering
\caption{Top-5 negatively interfering pairs --- Meta-Llama-3.1-8B on MBPP.}
\label{tab:r3}
\begin{tabular}{llrrrr}
\toprule
Prompt $i$ & Prompt $j$ & $\langle g_i,g_j\rangle$ & $\cos(g_i,g_j)$ & $\hat{p}_1(i)$ & $\hat{p}_1(j)$ \\
\midrule
mbpp\_792 & mbpp\_800 & $-3.71 \times 10^{-5}$ & $-0.113$ & 0.906 & 0.875 \\
mbpp\_517 & mbpp\_800 & $-2.32 \times 10^{-5}$ & $-0.072$ & 0.844 & 0.875 \\
mbpp\_366 & mbpp\_800 & $-1.46 \times 10^{-5}$ & $-0.055$ & 0.875 & 0.875 \\
mbpp\_792 & mbpp\_833 & $-1.25 \times 10^{-5}$ & $-0.045$ & 0.906 & 0.906 \\
mbpp\_329 & mbpp\_915 & $-9.84 \times 10^{-6}$ & \textbf{$-0.286$} & 0.938 & 0.031 \\
\bottomrule
\end{tabular}
\end{table}

\end{document}